%% file: main.tex
\documentclass[journal]{cls/IEEEtran}

\usepackage[table,usenames,dvipsnames]{xcolor}      
\usepackage{enumitem}                               
\usepackage{extarrows}                              
\usepackage[noadjust]{cite}                         

\usepackage{amsmath,amssymb,amsfonts,amsthm,dsfont} 

\usepackage{graphicx,tabularx,subcaption}
\usepackage[export]{adjustbox}
\usepackage{booktabs}
\usepackage{multirow}

\usepackage[font={small}]{caption}

\usepackage[titletoc,title]{appendix}

\usepackage[breaklinks=true, colorlinks, bookmarks=true, citecolor=Black, urlcolor=Violet,linkcolor=Black]{hyperref}

\def\argmin{\mathop{\arg\min}\limits}

\newcommand{\indicator}{\mathds{1}}

\DeclareMathOperator*{\adj}{adj}

\DeclareMathOperator*{\diag}{diag}
\DeclareMathOperator*{\sgn}{sgn}
\DeclareMathOperator*{\vectorize}{vec}

\newcommand{\scaleMathLine}[2][1]{\resizebox{#1\linewidth}{!}{$\displaystyle{#2}$}}

\newcommand{\prl}[1]{\left(#1\right)}
\newcommand{\brl}[1]{\left[#1\right]}
\newcommand{\crl}[1]{\left\{#1\right\}}

\newcommand{\ubft}{\underline{\bft}}

\newcommand{\ubfb}{\underline{\bfb}}

\theoremstyle{definition}

\newtheorem*{definition*}{Definition}
\newtheorem*{problem*}{Problem}

\newtheorem*{proposition*}{Proposition}
\newtheorem{proposition}{Proposition}

\newtheorem{lemma}{Lemma}

\input{cls/sym.tex}

\begin{document}
%
\title{OrcVIO: Object residual constrained Visual Inertial Odometry}
%
%
%

\author{Mo Shan,~\IEEEmembership{Student Member,~IEEE,}
        Vikas Dhiman,~\IEEEmembership{Member,~IEEE,}
        Qiaojun Feng,~\IEEEmembership{Student Member,~IEEE,}
        Jinzhao Li,
        and Nikolay Atanasov,~\IEEEmembership{Member,~IEEE}
\thanks{We gratefully acknowledge support from ARL DCIST CRA W911NF-17-2-0181. The authors are with the Department of Electrical and Computer Engineering, University of California San Diego, La Jolla, CA 92093, USA {\tt\footnotesize \{moshan,vdhiman,qjfeng,jil016,natanasov\}@ucsd.edu}.}%
}

%
%



\maketitle

\begin{abstract}
Introducing object models in the variables optimized by simultaneous localization and mapping (SLAM) algorithms not only improves performance but also supports specification and execution of semantically meaningful robotic tasks. This work presents an Object residual constrained Visual Inertial Odometry (OrcVIO) algorithm for online sensor localization, tightly coupled with 3D object pose and shape estimation. OrcVIO initializes and optimizes the pose and shape of detected and tracked objects by differentiating through two new optimization terms capturing object-feature and bounding-box measurements. The estimated object poses and shapes aid in real-time incremental multi-state constraint Kalman filtering (MSCKF) over the visual-inertial sensor states. The ability of OrcVIO for accurate sensor trajectory estimation and large-scale object mapping is evaluated using simulated and real data.
\end{abstract}


%
\IEEEpeerreviewmaketitle

\input{tex/Supp}

\input{tex/Introduction.tex}

\input{tex/Review.tex}

\input{tex/Background.tex}

\input{tex/Problem.tex}

\input{tex/Reconstruction.tex}

\input{tex/OrcVIO.tex}

\input{tex/Evaluation.tex}

\input{tex/Conclusion.tex}

\input{tex/Appendix.tex}


{\small
\bibliographystyle{cls/IEEEtran}
\bibliography{bib/ref.bib}
}

%

\begin{IEEEbiography}[{\includegraphics[width=1in,height=1.25in,clip,keepaspectratio]{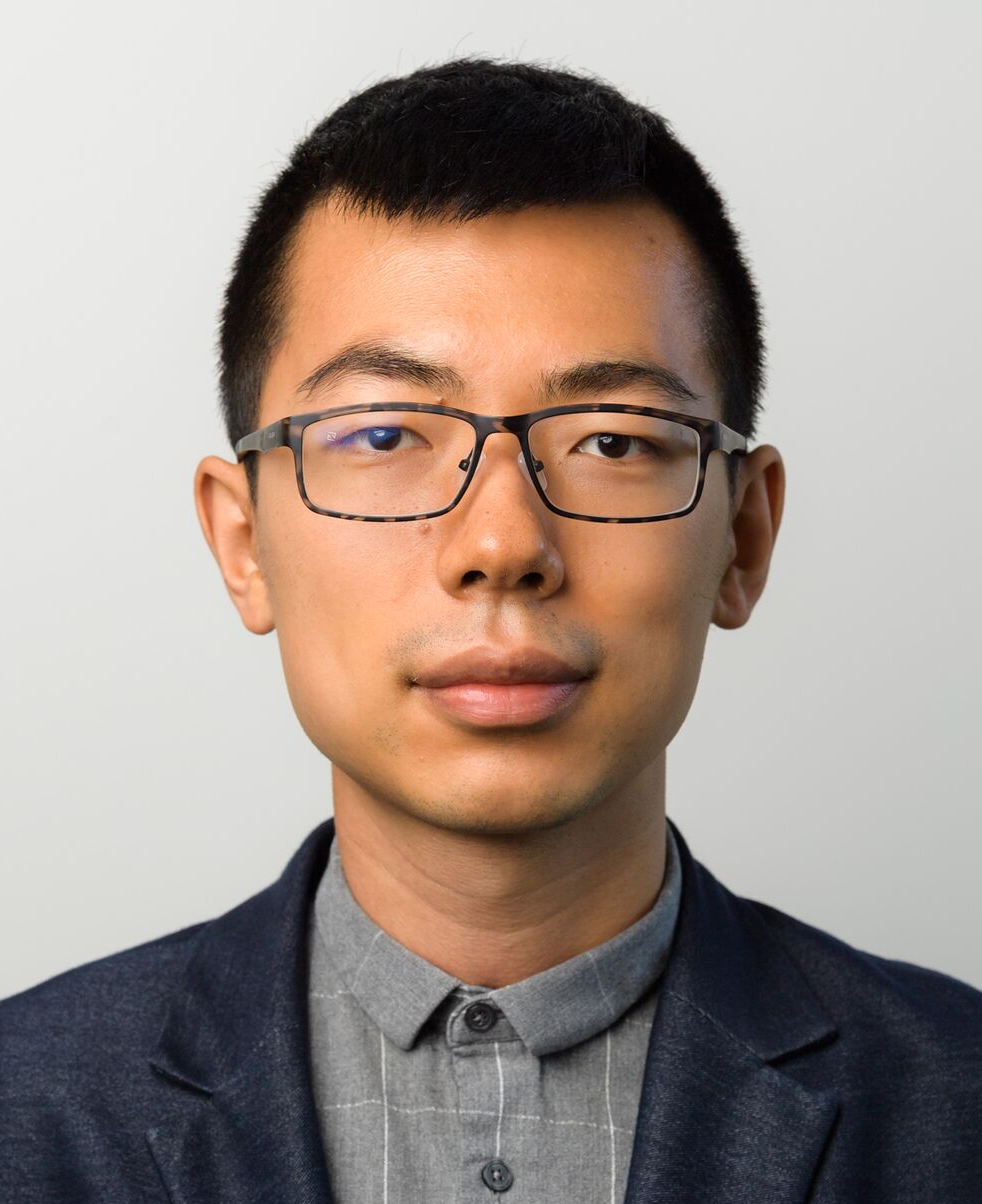}}]{Mo Shan} (S'20) is a PhD student of Electrical and Computer Engineering at the University of California San Diego. He obtained a B.S. degree in Electrical Engineering from National University of Singapore, Singapore, in 2014 and M.S. degree in Electrical Engineering from the University of California San Diego. His research focuses on semantic SLAM and visual inertial odometry.
\end{IEEEbiography}

\begin{IEEEbiography}
[{\includegraphics[width=1in,height=1.25in,clip,keepaspectratio]{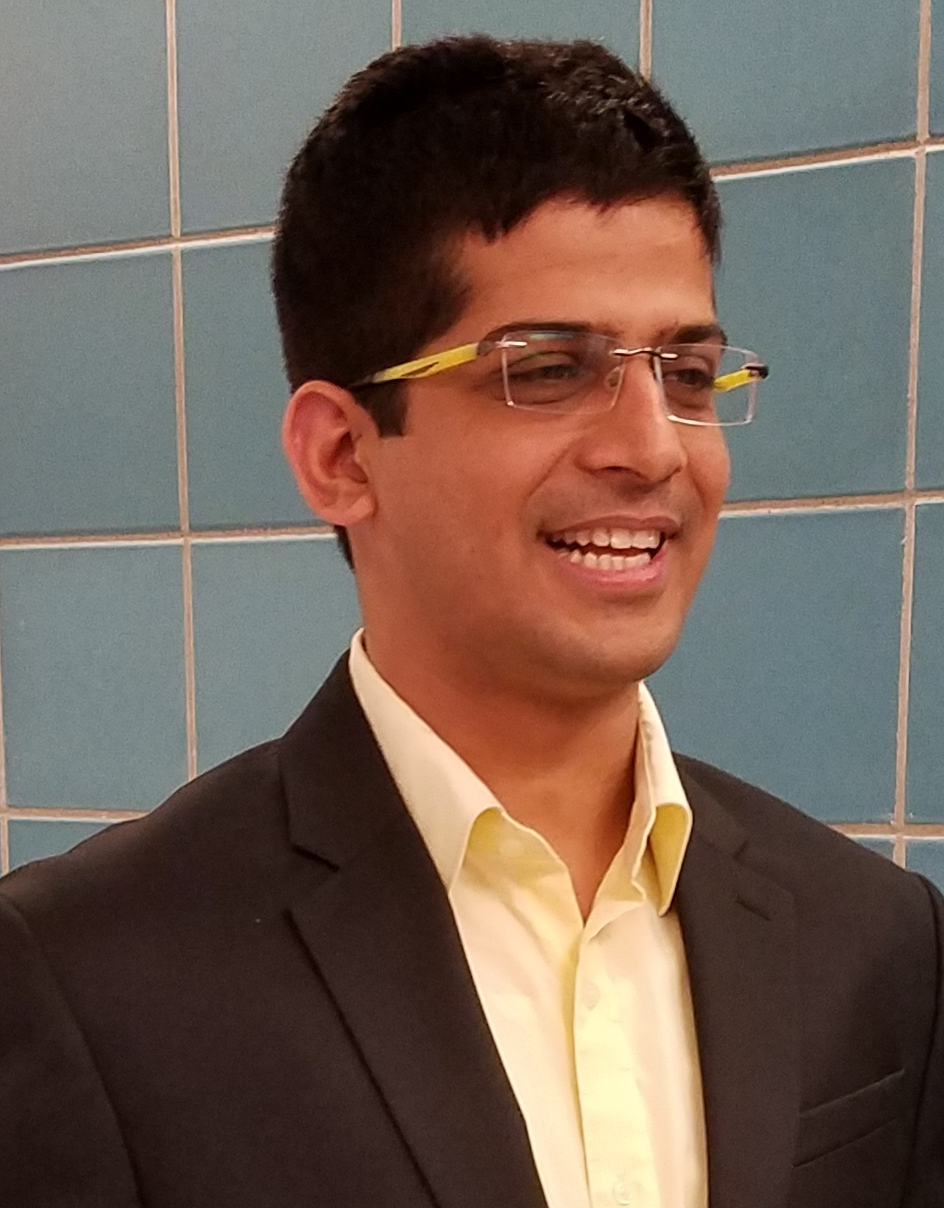}}]{Vikas Dhiman} (M'15) is a Postdoctoral Researcher at the Contextual Robotics Institute and the Existential Robotics Lab at the University of California, San Diego. His work focuses on localization, mapping, and control algorithms for robotics applications. He graduated in Electrical Engineering (2008) from Indian Institute of Technology, Roorkee, earned an M.S. (2014) from the University at Buffalo, and received a Ph.D. (2019) from the University of Michigan, Ann Arbor. 
\end{IEEEbiography}


\begin{IEEEbiography}[{\includegraphics[width=1in,height=1.25in,clip,keepaspectratio]{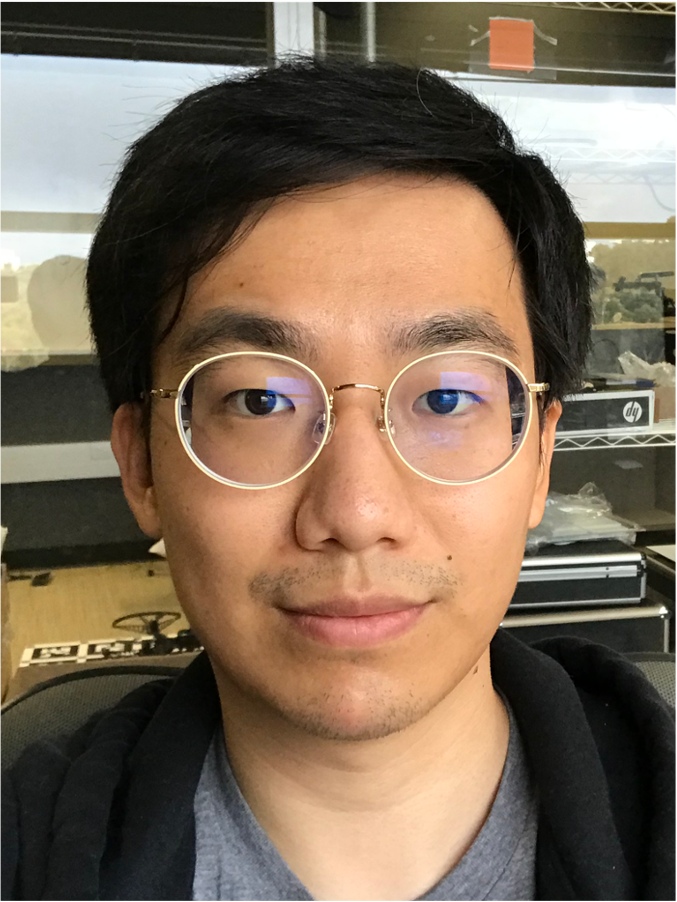}}]{Qiaojun Feng} (S'19) obtained a B.Eng. in Automation from Tsinghua University, Beijing, China, in 2017 and an M.S. in Electrical Engineering from the University of California San Diego, La Jolla, CA, in 2019. He is currently a Ph.D. student of Electrical and Computer Engineering at the University of California San Diego. His research focuses on robotics and autonomy, especially SLAM. He works on environment perception and map reconstruction, combining geometric and semantic information.
\end{IEEEbiography}

\begin{IEEEbiography}[{\includegraphics[width=1in,height=1.25in,clip,keepaspectratio]{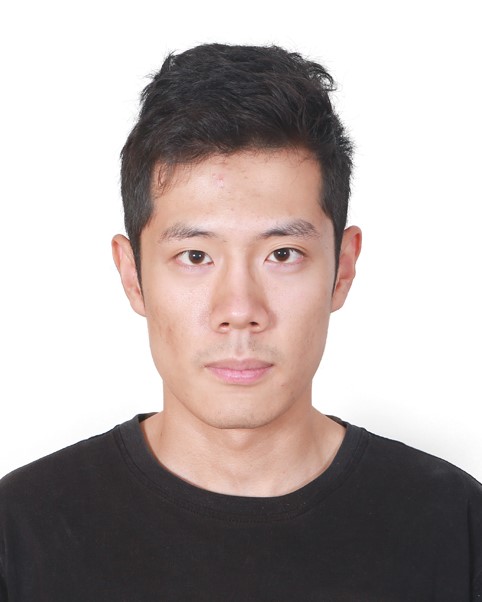}}]{Jinzhao Li} obtained a B.Eng. in Automation from Tsinghua University, Beijing, China, in 2018 and an M.S. in Electrical and Computer Engineering from the University of California San Diego, La Jolla, CA, in 2020. His research focuses on SLAM, visual inertial odometry, and object-level reconstruction.
\end{IEEEbiography}

\begin{IEEEbiography}[{\includegraphics[width=1in,height=1.25in,clip,keepaspectratio]{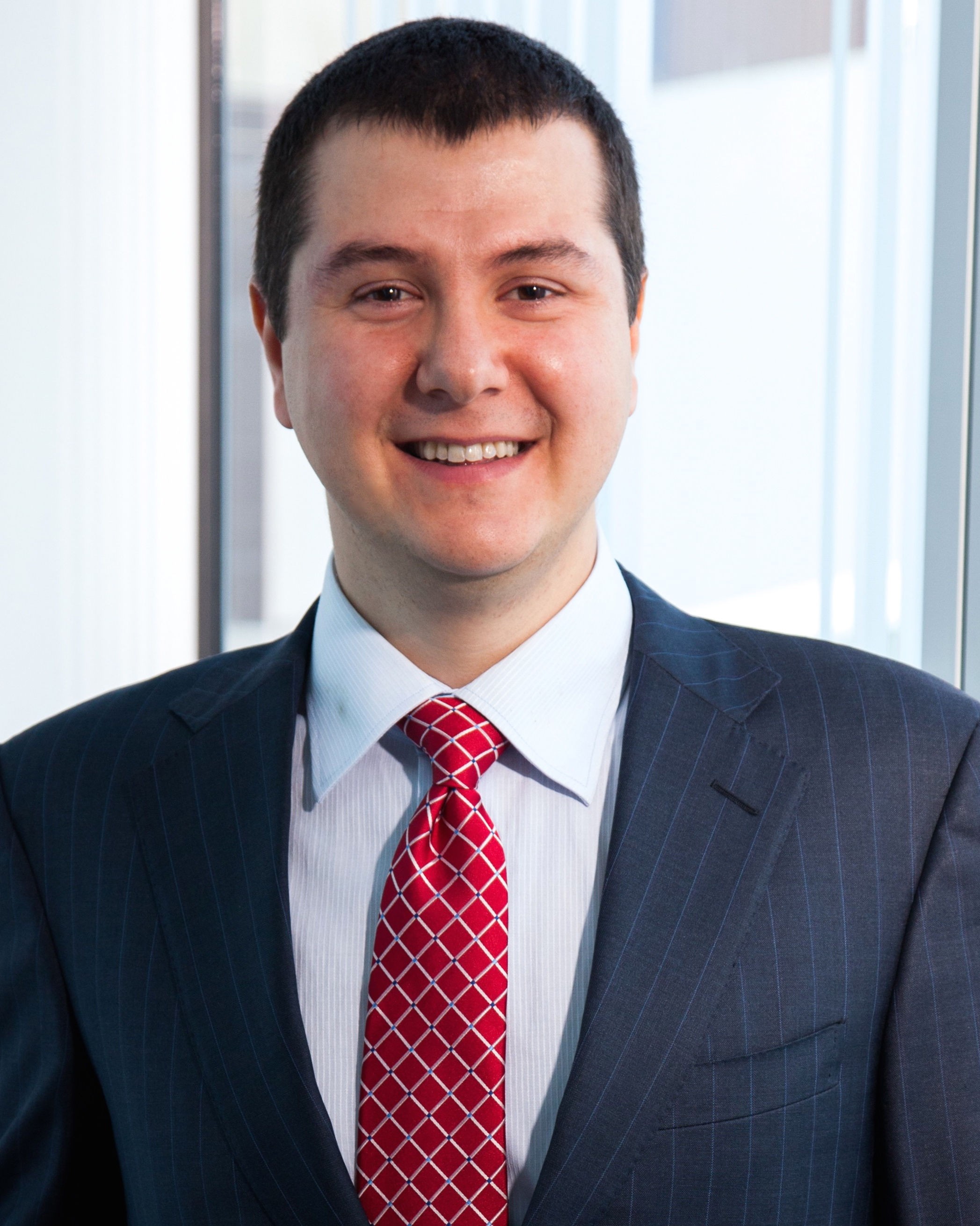}}]{Nikolay Atanasov}
(S'07-M'16) is an Assistant Professor of Electrical and Computer Engineering at the University of California San Diego. He obtained a B.S. degree in Electrical Engineering from Trinity College, Hartford, CT, in 2008 and M.S. and Ph.D. degrees in Electrical and Systems Engineering from the University of Pennsylvania, Philadelphia, PA, in 2012 and 2015, respectively. His research focuses on robotics, control theory, and machine learning, applied to active sensing using ground and aerial robots. He works on probabilistic environment models that unify geometry and semantics and on optimal control and reinforcement learning approaches for minimizing uncertainty in these models. Dr. Atanasov's work has been recognized by the Joseph and Rosaline Wolf award for the best Ph.D. dissertation in Electrical and Systems Engineering at the University of Pennsylvania in 2015 and the best conference paper award at the IEEE International Conference on Robotics and Automation in 2017.
\end{IEEEbiography}




\end{document}

%% file: cls/sym.tex
\newcommand{\calE}{{\cal E}}

\newcommand{\calL}{{\cal L}}

\newcommand{\calO}{{\cal O}}

\newcommand{\calX}{{\cal X}}

\newcommand{\bfa}{\mathbf{a}}
\newcommand{\bfb}{\mathbf{b}}
\newcommand{\bfc}{\mathbf{c}}

\newcommand{\bfe}{\mathbf{e}}

\newcommand{\bfg}{\mathbf{g}}

\newcommand{\bfi}{\mathbf{i}}

\newcommand{\bfn}{\mathbf{n}}
\newcommand{\bfo}{\mathbf{o}}
\newcommand{\bfp}{\mathbf{p}}

\newcommand{\bfs}{\mathbf{s}}
\newcommand{\bft}{\mathbf{t}}
\newcommand{\bfu}{\mathbf{u}}
\newcommand{\bfv}{\mathbf{v}}
\newcommand{\bfw}{\mathbf{w}}
\newcommand{\bfx}{\mathbf{x}}
\newcommand{\bfy}{\mathbf{y}}
\newcommand{\bfz}{\mathbf{z}}

\newcommand{\bfzeta}{\boldsymbol{\zeta}}

\newcommand{\bftheta}{\boldsymbol{\theta}}

\newcommand{\bfpi}{\boldsymbol{\pi}}
\newcommand{\bfrho}{\boldsymbol{\rho}}

\newcommand{\bfomega}{\boldsymbol{\omega}}
\newcommand{\bfxi}{\boldsymbol{\xi}}
\newcommand{\bfell}{\boldsymbol{\ell}}

\newcommand{\bfA}{\mathbf{A}}
\newcommand{\bfB}{\mathbf{B}}

\newcommand{\bfF}{\mathbf{F}}

\newcommand{\bfH}{\mathbf{H}}
\newcommand{\bfI}{\mathbf{I}}
\newcommand{\bfJ}{\mathbf{J}}
\newcommand{\bfK}{\mathbf{K}}

\newcommand{\bfN}{\mathbf{N}}

\newcommand{\bfP}{\mathbf{P}}
\newcommand{\bfQ}{\mathbf{Q}}
\newcommand{\bfR}{\mathbf{R}}

\newcommand{\bfT}{\mathbf{T}}
\newcommand{\bfU}{\mathbf{U}}
\newcommand{\bfV}{\mathbf{V}}

\newcommand{\bfY}{\mathbf{Y}}

\newcommand{\bfSigma}{\boldsymbol{\Sigma}}

\newcommand{\bfPhi}{\boldsymbol{\Phi}}

\newcommand{\bbE}{\mathbb{E}}

\newcommand{\bbR}{\mathbb{R}}

%% file: tex/Supp.tex
\section*{Supplementary Material}
\label{sec:supp}

\noindent Software and videos supplementing this paper are available at:
\centerline{\url{http://erl.ucsd.edu/pages/orcvio.html}}

%% file: tex/Introduction.tex
\section{Introduction}
\label{sec:introduction}



The foundations of visual environment understanding in robotics, machine learning, and computer vision lie in the twin technologies of inferring geometric structure and semantic content. Researchers have made significant progress in geometric structure reconstruction using Structure from Motion (SfM) \cite{agarwal2009building, schonberger2016structure} and SLAM \cite{cadena2016past} techniques. State of the art SLAM approaches work with monocular or stereo cameras \cite{orbslam2}, often complemented by inertial information \cite{msckf,vinsmono}. However, most real-time incremental SLAM techniques provide purely geometric representations, e.g., of points, lines, or planes, that lack semantic interpretation of the environment. 

Recently, tremendous progress has been achieved in semantic scene understanding using deep neural networks for object detection \cite{bochkovskiy2020yolov4}, instance segmentation \cite{he2017mask}, and object tracking \cite{Wojke2018deep}. Nevertheless, the literature in deep learning is sparse in techniques that provide global positioning of the detected and tracked objects to obtain an object-level map online.


\begin{figure}[t]
  \centering
  \includegraphics[width=\linewidth]{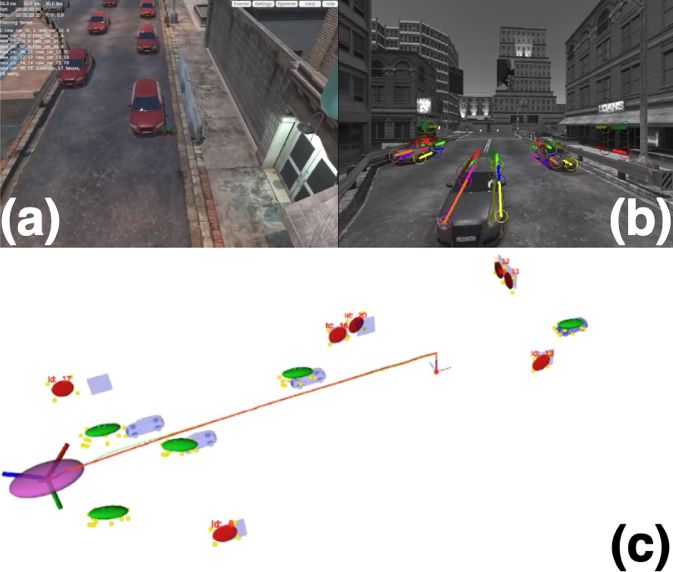}
  \caption{This work develops a tightly coupled visual-inertial odometry and object state estimation algorithm. The figure shows: (a) a quadrotor robot navigating in a simulated Unity environment, containing cars and doors, (b) color-coded semantic keypoint features detected and tracked on the cars and doors, and (c) ground-truth (red) and estimated (green) robot trajectory as well as ground-truth cars and doors (blue meshes) and estimated cars (green ellipsoids) and doors (red ellipsoids) along with their semantic keypoints (yellow points). A video demonstration can be found in the supplementary material.}
  \label{fig:overview}
\end{figure}


%

This paper focuses on joint visual-inertial odometry and object mapping, bridging the gap between geometric and semantic inference in SLAM. Generating geometrically consistent and semantically meaningful maps allows compressed representation, improved loop closure (recognizing already visited locations), and robot mission specifications in terms of human-interpretable objects, e.g., for safe navigation, manipulation, multi-stage interaction \cite{lorbach2014prior, sahin2020review, vasilopoulos2020reactive, garg2020semantics}. Our work introduces object states, modeling the position, orientation, and shape of object instances in the environment. We utilize a coarse category-agnostic and a fine category-specific model of object shape. The coarse model uses an ellipsoid to restrict an object's pose variation and relate its coarse shape to bounding-box detections. The fine model uses a set of mid-level object parts (e.g., car wheels, windshield, doors), called \emph{semantic keypoints}, to obtain a precise part-based shape description. Given streaming inertial measurement unit (IMU) and monocular camera measurements, we develop an algorithm to simultaneously estimate the IMU-camera trajectory and the states of the objects, detected and tracked in the camera images. The \textbf{contributions} of this paper are summarized as follows.
\begin{itemize}

\item We introduce object states in the formulation of a SLAM problem, modeling position, orientation, coarse ellipsoid shape, and fine semantic-keypoint shape.

\item We define residuals relating object states and IMU-camera states to inertial measurements, geometric features, object semantic features, and object bounding-box detections, and explicitly derive their Jacobians.

\item We develop an extension of the multi-state constraint Kalman filter (MSCKF) \cite{msckf} to enable online tightly coupled estimation of object and IMU-camera states. Our innovations include closed-form mean and covariance propagation over the SE(3) pose and velocity manifold of the IMU-camera states and measurement updates based on our new semantic residuals with object states optimized over multiple views. Our algorithm is suitable for real-time incremental odometry and object mapping and is more efficient than nonlinear batch optimization.


\end{itemize}
We name our method \emph{Object residual constrained Visual Inertial Odometry} (OrcVIO) to emphasize the role of the semantic residuals in the optimization process. OrcVIO is capable of producing meaningful object maps and estimating accurate sensor trajectories, as shown in Fig. \ref{fig:overview}.

%% file: tex/Review.tex
\section{Related Work}
\label{sec:Review}

Many visual SLAM approaches work with monocular \cite{msckf,civera2008inverse, civera20101, svo2, vinsmono} or stereo cameras \cite{howard2008real, orbslam2, smsckf}. Featureless approaches \cite{newcombe2011dtam, engel2014lsd, gao2018ldso} that minimize image intensity directly have been proposed, and inertial information is often used to complement the visual information \cite{msckf, smsckf, weiss2011real,wu2015square,okvis}. The MSCKF algorithm \cite{msckf} leverages both inertial data and visual features in an extened Kalman filter formulation. Each visual feature whose track is lost provides multi-frame constraints for a corresponding 3D landmark. The constraint residuals are linearized to perform the filter update step. A key idea is to marginalize the landmark states via projection to the null space of the visual feature Jacobians, allowing \emph{structureless} estimation of the IMU-camera states only. A key extension of OrcVIO over the MSCKF and visual-inertial SLAM algorithms in general is to introduce object measurements (bounding boxes and semantic features) and object states whose residuals constrain the IMU-camera states.


Recent methods have considered learning to regress camera poses directly from raw images \cite{kendall2015posenet, zhou2017unsupervised, clark2017vinet, clark2017vidloc, wang2017deepvo, yin2018geonet, li2018undeepvo}. For instance, monocular depth, optical flow, and ego-motion are jointly optimized from video in \cite{yin2018geonet} by relying on a view-synthesis loss. These unsupervised learning techniques have shown impressive performance in localization but do not generate global maps. 
In this paper, we focus on obtaining a joint geometric-semantic representations from measurements in real time, i.e., spatial perception \cite{cao2020representations}. Prior works that utilize both spatial and semantic information include \cite{galindo2005multi, leibe2007dynamic, civera2011towards, pronobis2011semantic, stuckler2015dense, vineet2015incremental, pillai2015monocular, pire2019online}, but the spatial and semantic states are estimated independently and merged later. 
On the other hand, \cite{atanasov2014semantic, reid2014towards, kundu2014joint, galvez2016real, doherty2020probabilistic,       rosinol2020kimera} consider joint metric and semantic mapping. 
Recent works focus on the tightly coupled spatial and semantic estimation, and there are mainly two groups of object-based SLAM techniques: category-specific and category-agnostic. 

Category-specific approaches optimize the pose and shape of object instances, using semantic keypoints \cite{Krishna_ICRA2017,parkhiya2018constructing} or 3D shape models \cite{salas2013slam++,semslam,Atanasov_SemanticSLAM_IJCAI18,dongFS17,fei2018visual,hu2018dense,feng2019localization,hu2019deep,ishimtsev2020cad}. 
For example, \cite{salas2013slam++} introduces a real-time joint 3D object pose and camera pose estimation via pose graph optimization. The objects stored in a database that are also present in the current frame are detected and optimized, using the vertex and normal map from a RGBD sensor. Object pose and shape are optimized in \cite{parkhiya2018constructing} using semantic keypoints to provide additional error terms related to object pose in the SLAM factor graph. Visual-inertial information for object mapping is used in \cite{fei2018visual}, relying on a database to retrieve object shapes. Hu et al.~\cite{hu2018dense} embed object priors into least-squares minimization to incrementally track and map chairs. The object shape represented by a binary voxel grid is compactly described by a latent code obtained from an auto-encoder, which is used for shape initialization and iterative residual minimization. 
These methods in general are computationally demanding due to iterative batch optimization and reliance on instance-specific CAD models.

Category-agnostic approaches use geometric shapes, such as spheres \cite{frost2018recovering,okhierarchical}, cuboids \cite{bao2011semantic,cubeslam}, or ellipsoids \cite{dhiman2016continuous,rubino20173d,GAY2018124,quadric_slam,hosseinzadeh2018structure,ok2019robust,wu2020eao,liao2020object}, to represent objects. SSFM \cite{bao2011semantic} uses the tightest bounding cube enclosing an object to parameterize the object location and pose. The object measurements include the location, size of the object bounding box and the object pose obtained from a 3D object detector~\cite{savarese20073d}. Assuming the semantic measurements are consistent across frames, an object state is optimized via maximum likelihood estimation. CubeSLAM \cite{cubeslam} generates and refines 3D cuboid proposals using multi-view bundle adjustment without relying on prior models. QuadricSLAM \cite{quadric_slam} uses an ellipsoid representation, suitable for defining a bounding-box detection model. Structural constraints based on supporting and tangent planes, commonly observed under a Manhattan assumption, have also been introduced \cite{quadrics_reid}. Using generic symmetric shapes, however, makes the orientation of object instances potentially irrecoverable. For instance, the front and back of an object become indistinguishable.

This paper extends our prior conference publication \cite{orcvio} with several theoretical contributions and a large-scale experimental evaluation of OrcVIO. While \cite{orcvio} used linear velocity measurements, this paper uses a complete description of all IMU states for odometry and derives a novel closed-form expression for covariance propagation. A new bounding-box residual is developed to ensure that it scales equivalently as the keypoint residual. In \cite{orcvio}, the bounding-box residual was quadratic instead of linear as a function of the bounding-box lines. This version also implements a zero-velocity residual to reduce VIO drift when the sensor is completely static. Extensive evaluation is conducted on the KITTI dataset, a photo-realistic Unity dataset, and in indoor and outdoor real-time experiments. We also extend OrcVIO to handle multiple object classes, including cars, doors, barriers, monitors, chairs.


%% file: tex/Background.tex
\section{Background and Notation}
\label{sec:background}

We denote the IMU, camera, object, and global reference frames as $\{I\}$, $\{C\}$, $\{O\}$, $\{G\}$, respectively. The transformation from frame $\{A\}$ to $\{B\}$ is specified by a $4\times 4$ matrix:
\begin{equation}
{}^B_A\bfT \triangleq \begin{bmatrix} {}^B_A\bfR & {}^B_A\bfp\\\mathbf{0}^\top & 1 \end{bmatrix} \in SE(3)
\end{equation}
where ${}^B_A\bfR \in SO(3)$ is a rotation matrix and ${}^B_A\bfp \in \mathbb{R}^3$ is a translation vector. To simplify the notation, we will not explicitly indicate the global frame when specifying transformations. For example, the pose of the IMU frame $\{I\}$ in $\{G\}$ at time $t_k$ is specified by ${}_I\bfT_k$.

We overload $(\cdot)_\times$ to denote the mapping from an axis-angle vector $\bftheta \in \mathbb{R}^3$ to a $3 \times 3$ skew-symmetric matrix $\bftheta_{\times} \in \mathfrak{so}(3)$ as well as from a vector $\bfxi \in \mathbb{R}^6$ to a $4 \times 4$ twist matrix:
\begin{equation}
\bfxi = \begin{bmatrix} \bfrho\\\bftheta \end{bmatrix} \in \bbR^6 \qquad \bfxi_\times := \begin{bmatrix} \bftheta_{\times} & \bfrho \\ \mathbf{0}^\top & 0 \end{bmatrix} \in \mathfrak{se}(3).
\end{equation}
We define an infinitesimal change of pose $\bfT \in SE(3)$ using a right perturbation $\bfT\exp\prl{\bfxi_\times} \in SE(3)$ (see \cite[Ch.7]{BarfootBook}).


Let $\underline{\bfx} = \begin{bmatrix}\bfx^\top & 1\end{bmatrix}^\top$ be the homogeneous coordinates of a vector $\bfx$. For $\bfx \in \mathbb{R}^3$, we define the operators $\underline{\bfx}^\odot$ and $\underline{\bfx}^\circledcirc$:
%
\begin{equation} \label{eq:odot}
\scaleMathLine[0.9]{\underline{\bfx}^\odot \triangleq \begin{bmatrix} \bfI_3 & -\bfx_\times\\ \mathbf{0}^\top & \mathbf{0}^\top \end{bmatrix} \in \mathbb{R}^{4 \times 6}, \;\; \underline{\bfx}^\circledcirc \triangleq \begin{bmatrix} \mathbf{0} & \bfx\\ -\bfx_\times & \mathbf{0} \end{bmatrix} \in \mathbb{R}^{6 \times 4},}
\end{equation}
where $\bfI_3$ is the $3 \times 3$ identity matrix. A \emph{quadric shape} \cite[Ch.3]{MVGBook} is a set $\crl{ \bfx \mid \underline{\bfx}^\top \bfQ \underline{\bfx} \leq 0}$, where $\bfQ$ is a symmetric matrix. Consider an axis-aligned ellipsoid centered at $\mathbf{0}$:
\begin{equation}
\mathcal{E}_{\bfu} \triangleq \crl{\bfx \mid \bfx^\top \bfU^{-\top}\bfU^{-1}\bfx \leq 1},
\end{equation}
where $\bfU \triangleq \diag(\bfu)$ and the elements of the vector $\bfu$ are the lengths of the semi-axes of $\mathcal{E}_{\bfu}$. In homogeneous coordinates, $\mathcal{E}_{\bfu}$ is a special case of a quadric shape $\crl{ \bfx \mid \underline{\bfx}^\top \bfQ_{\bfu} \underline{\bfx} \leq 0}$ with $\bfQ_{\bfu} \triangleq \mathbf{diag}(\bfU^{-2},-1)$. A quadric shape can also be defined in dual form, as the set of planes $\underline{\bfpi} = \bfQ \underline{\bfx}$ that are tangent to the shape surface at each $\bfx$. A \emph{dual quadric surface} is defined as $\crl{ \bfpi \mid \underline{\bfpi}^\top \bfQ^* \underline{\bfpi} = 0}$, where $\bfQ^* = \adj(\bfQ)$\footnote{If $\bfQ$ is invertible, $\bfQ^* = \adj(\bfQ) = \det(\bfQ)\bfQ^{-1}$ can be simplified to $\bfQ^* = \bfQ^{-1}$ due to the scale-invariance of the dual quadric surface definition.}. A dual quadric surface defined by $\bfQ^* \in \bbR^{4 \times 4}$ can be transformed by $\bfT \in SE(3)$ to another reference frame as $\bfT \bfQ^* \bfT^\top$. Similarly, it can be projected to a lower-dimensional space by a projection matrix $\bfP \triangleq \begin{bmatrix} \bfI & \mathbf{0}\end{bmatrix}$ as $\bfP \bfQ^* \bfP^\top$.


%% file: tex/Problem.tex
\begin{figure}[t]
  \centering
  \includegraphics[width=\linewidth]{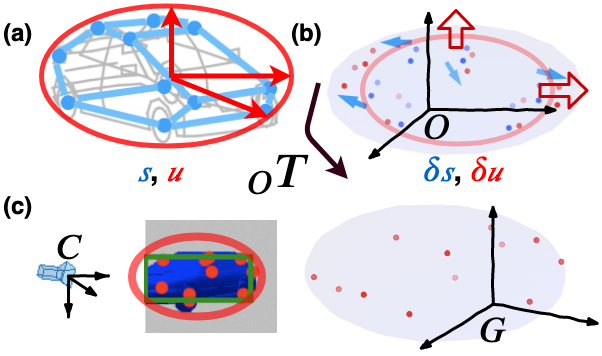}
  \caption{(a) An object class is defined by a semantic class $\sigma$ and average shape specified by semantic landmarks $\bfs$ (blue) and an ellipsoid with semi-axes lengths $\bfu$ (red). 
  (b) A specific instance has landmark and ellipsoid deformations parameterized by $\delta\bfs$ (blue arrows) and $\delta\bfu$ (red arrows). (c) The landmarks and ellipsoid are transformed from the object frame $\{O\}$ to the global frame $\{G\}$ via the instance pose ${}_O\bfT$.}
  \label{fig:tsdv}
\end{figure}

\begin{figure*}
\centering
\includegraphics[width=\linewidth]{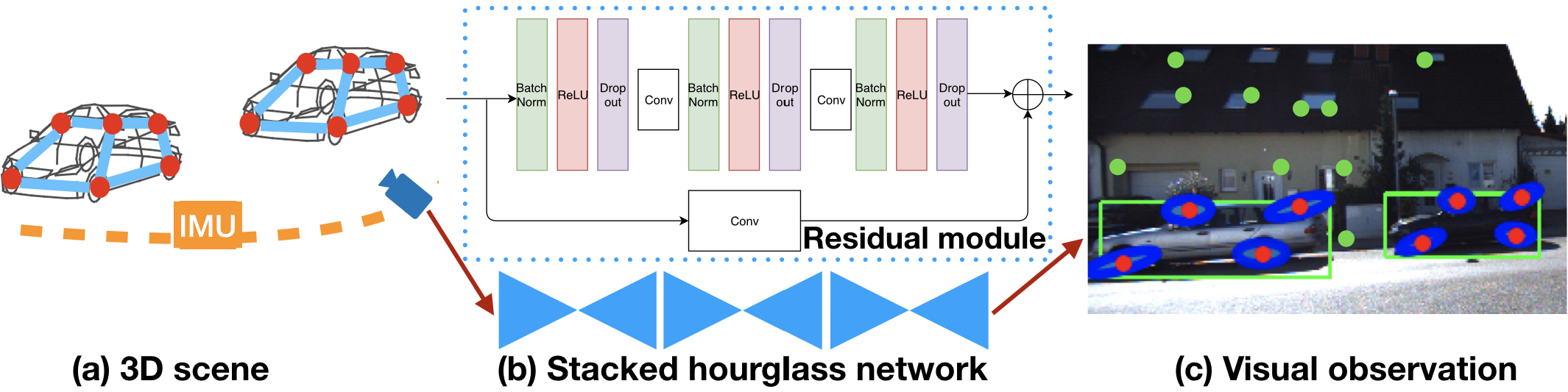}
\caption{OrcVIO utilizes visual-inertial information to optimize the sensor trajectory and the shapes and poses of objects. The observations include geometric keypoints (FAST keypoints indicated by green dots in (c)), semantic keypoints (car parts indicated by red dots in (c)), semantic keypoint covariances (blue ellipses in (c)), bounding-boxes (green boxes in (c)), and inertial data (orange dotted lines in (a)). The semantic keypoints and their covariances are obtained from a Bayesian stacked hourglass CNN in (b) composed of residual modules (blue dotted rectangle in (b)), including convolution, ReLU, batch normalization, and dropout layers. The dropout layers are used to sample different weight realizations to enable a test-time estimate of the semantic keypoint covariances.}
\label{fig:overview_all}
\end{figure*}

\section{Problem Formulation}
\label{sec:problem_formulation}

Consider a system equipped with an IMU-camera sensor. Let ${}_{I}\bfx_k \triangleq \prl{{}_{I}\bfR_k,\; {}_{I}\bfv_k,\; {}_{I}\bfp_k,\; \bfb_{g,k},\; \bfb_{a,k}}$ be the IMU state at time $t_k$, consisting of orientation ${}_{I}\bfR_k \in SO(3)$, velocity ${}_{I}\bfv_k \in \mathbb{R}^3$, position ${}_{I}\bfp_k \in \mathbb{R}^3$, gyroscope bias $\bfb_{g,k} \in \mathbb{R}^3$, and accelerometer bias $\bfb_{a,k} \in \mathbb{R}^3$. Assume that the camera is rigidly attached to the IMU with relative transformation ${}^I_C\bfT \in SE(3)$, known from extrinsic calibration. Given ${}_{I}\bfx_k$, the camera pose can be obtained as ${}_C\bfT_k = {}_I\bfT_k {}^I_C\bfT$, where ${}_I\bfT_k \in SE(3)$ is the IMU pose. To facilitate the use of multi-frame camera information, define an augmented state $\bfx_k \triangleq \prl{{}_{I}\bfx_k,\; {}_{I} \bfT_{k-1},\;\ldots,\; {}_{I} \bfT_{k-W}}$, containing a sliding window of $W$ past IMU poses in addition to the current IMU state ${}_{I}\bfx_k$. The system state over time is $\mathcal{X} \triangleq \crl{\bfx_k}_k$.

Suppose that the system evolves in an unknown environment, containing \emph{geometric landmarks} $\calL \triangleq \crl{\bfell_m}_m$ and \emph{objects} $\calO \triangleq \crl{\bfo_i}_i$, represented in a global frame $\{G\}$. A geometric landmark is a static point $\bfell_m \in \mathbb{R}^3$, detectable via image keypoint algorithms, such as FAST \cite{FAST}. An object $\bfo_i = (\bfc_i,\bfi_i)$ is an instance $\bfi_i$ of a semantic class $\bfc_i$, detectable via object recognition algorithms, such as YOLO \cite{redmon2018yolov3}. The precise definitions of object class and instance follow.

\begin{definition*}
An \emph{object class} is a tuple $\bfc \triangleq \prl{\sigma, \bfs, \bfu}$, where $\sigma \in \mathbb{N}$ specifies a semantic type (e.g., chair, table, monitor) and $\bfs \in \mathbb{R}^{3 \times N_s}$ and $\bfu \in \mathbb{R}^3$ specify the average class shape. The class shape is modeled by an axis-aligned ellipsoid $\calE_{\bfu}$ and a set of \emph{semantic landmarks} $\bfs_l \in \mathbb{R}^3$ corresponding to the columns of $\bfs$. The semantic landmarks $\bfs_l$ define the 3D positions of mid-level object parts\footnote{Category-level object parts, such as the front wheel of a car, may be detected by a stacked hourglass convolutional network \cite{zhou2018starmap}.} in the object class canonical frame $\{O\}$.
\end{definition*}


\begin{definition*}
An \emph{object instance} of class $\bfc$ is a tuple $\bfi \triangleq \prl{{}_O\bfT, \delta\bfs, \delta\bfu}$, where ${}_O\bfT \in SE(3)$ is the instance pose and $\delta\bfs \in \mathbb{R}^{3 \times N_s}$, $\delta\bfu \in \mathbb{R}^3$ are deformations of the average class-level semantic landmarks $\bfs$ and ellipsoid semi-axes lengths $\bfu$.
\end{definition*}

The shape of an object $\bfo_i$ in the global frame $\{G\}$ is obtained by deforming and transforming the semantic landmark positions, ${}_O\bfT\prl{\underline{\bfs}_l + \underline{\delta\bfs}_l}$, and the dual ellipsoid, ${}_O\bfT\bfQ_{(\bfu + \delta\bfu)}^*{}_O\bfT^\top$, using the instance pose ${}_O\bfT$ and deformations $\delta\bfs$, $\delta\bfu$. Fig.~\ref{fig:tsdv} shows an illustration for a car model with $12$ semantic landmarks.



The IMU-camera sensor provides inertial measurements $^i\bfz_k$, geometric keypoint measurements $^g\bfz_{k,n}$, and semantic measurements, containing object class ${}^{c}_{}\mathbf{z}_{k,j}$, bounding-box ${}^{b}\mathbf{z}_{k,l,j}$, and semantic keypoint ${}^{s}_{}\mathbf{z}_{k,l,j}$ detections, illustrated in Fig.~\ref{fig:overview_all}. The inertial measurements ${}^{i}\mathbf{z}_k \triangleq \prl{{}^{i}\bfomega_k,{}^{i}\bfa_k}\in \mathbb{R}^6$ are the IMU's body-frame angular velocity $^i\bfomega_k$ and linear acceleration $^i\bfa_k$ at time $t_k$. The geometric keypoint measurements are noisy detections ${}^{g}_{}\mathbf{z}_{k,n} \in \mathbb{R}^{2}$ in normalized pixel coordinates\footnote{Given pixel coordinates $\bfz \in \mathbb{R}^2$ and a camera intrinsic calibration matrix $\bfK \in \mathbb{R}^{3 \times 3}$, the \emph{normalized pixel coordinates} of $\bfz$ are $\bfP\bfK^{-1} \underline{\bfz} \in \bbR^2$.} of the image projections of the subset of geometric landmarks $\calL$ visible to the camera at time $t$. To obtain semantic observations, an object detection algorithm is applied to the image at time $t_k$, followed by semantic keypoint extraction within each detected bounding-box (see Sec.~\ref{sec:tracking} for details). The $j$-th object detection provides the object class ${}^{c}_{}\mathbf{z}_{k,j} \in \mathbb{N}$, bounding box ${}^{b}\mathbf{z}_{k,l,j} \in \mathbb{R}^{2}$, described by $l = 1,\ldots,4$ lines in normalized pixel coordinates, and semantic keypoints ${}^{s}_{}\mathbf{z}_{k,l,j} \in \mathbb{R}^{2}$ in normalized pixel coordinates associated with the $l = 1,\ldots,N_s$ semantic landmarks\footnote{The semantic landmark-keypoint correspondence is provided by the semantic keypoint detector. Some landmarks may not be detected due to occlusion but we do not make this explicit for simplicity.}.


Let $\indicator_{k,m,n} \in \{0,1\}$ indicate whether the $n$-th geometric keypoint observed at time $t_k$ is associated with the $m$-th geometric landmark. Similarly, let $\indicator_{k,i,j} \in \{0,1\}$ indicate whether the $j$-th object detection at time $t_k$ is associated with the $i$-th object instance. This data association information can be obtained by keypoint and object tracking as described in Sec.~\ref{sec:tracking}.
Given the associations, we introduce error functions:
\begin{equation*}
\begin{aligned}
{}^i\bfe_{k,k+1} &\triangleq {}^i\bfe\prl{\bfx_k,\bfx_{k+1},{}^{i}\mathbf{z}_k} \; &{}^g\bfe_{k,m,n} &\triangleq {}^g\bfe\prl{\bfx_k, \bfell_m, {}^{g}\mathbf{z}_{k,n}}\\
{}^s\bfe_{k,i,l,j} &\triangleq {}^s\bfe\prl{\bfx_k, \bfo_i, {}^{s}\mathbf{z}_{k,l,j}} \; & {}^b\bfe_{k,i,l,j} &\triangleq{}^b\bfe\prl{\bfx_k, \bfo_i, {}^{b}\mathbf{z}_{k,l,j}}
\end{aligned}
\end{equation*}
for the inertial, geometric, semantic keypoint and bounding-box measurements, respectively, defined precisely in Sec.~\ref{sec:reconstruction}. We also introduce an object shape regularization error term ${}^r\bfe\prl{\bfo_i}$ to ensure that the instance deformations $(\delta\bfs, \delta\bfu)$ remain small, and consider the following problem.

\begin{problem*}
Determine the sensor trajectory $\mathcal{X}^*$, geometric landmarks $\calL^*$, and object states $\mathcal{O}^*$ that minimize the weighted sum of squared errors:
\begin{align}
\label{eq:problem}
\min_{\calX, \calL, \calO} \; &{}^iw\sum_k \| {}^i\bfe_{k,k+1} \|_{{}^i\bfV}^2 +{}^gw \sum_{k,m,n} \indicator_{k,m,n}  \| {}^g\bfe_{k,m,n}\|_{{}^g\bfV}^2 \notag\\
+{}^sw& \sum_{k,i,l,j} \indicator_{k,i,j} \| {}^s\bfe_{k,i,l,j} \|_{{}^s\bfV}^2 +{}^bw \sum_{k,i,l,j} \indicator_{k,i,j}  \| {}^b\bfe_{k,i,l,j} \|_{{}^b\bfV}^2\notag\\
+{}^rw& \sum_{i} \|{}^r\bfe\prl{\bfo_i}\|_2^2
\end{align}
where ${}^*w$ are positive constants determining the relative importance of the error terms and ${}^*\bfV$ are positive-definite matrices specifying the covariances associated with the inertial, geometric, semantic, and bounding-box measurements. A measurement covariance $\bfV$ defines a quadratic (Mahalanobis) norm $\|\bfe\|_{\bfV}^2 \triangleq  \bfe^\top\bfV^{-1}\bfe$.
\end{problem*}


Inspired by the MSCKF algorithm \cite{msckf}, we decouple the optimization over $\calL$ and $\calO$ from that over $\calX$ to design an efficient real-time algorithm. When a geometric-keypoint or object track is lost, we perform multi-view iterative optimization over the corresponding geometric landmark $\ell_m$ or object $\bfo_i$ based on the estimate of the latest IMU-camera state $\bfx_t$. The IMU-camera state is propagated using the inertial observations and updated using the optimized geometric landmark and object states and the geometric and semantic observations. This decoupling leads to potentially lower accuracy but higher efficiency compared to window or batch keyframe optimization techniques \cite{semslam}. Our approach is among the first to offer tight coupling between semantic information and geometric structure in visual-inertial odometry. Error functions and Jacobians derived in Sec.~\ref{sec:reconstruction} can be used for batch keyframe optimization in factor-graph formulation of object SLAM \cite{quadric_slam}.


%% file: tex/Reconstruction.tex
\section{Landmark and Object Reconstruction}
\label{sec:reconstruction}

Our approach consists for a front-end measurement generation stage and a back-end landmark state and sensor pose optimization stage. This section discusses the detection and tracking of geometric keypoint measurements $^g\bfz_{k,n}$ and object class ${}^{c}_{}\mathbf{z}_{k,j}$, bounding-box ${}^{b}\mathbf{z}_{k,l,j}$, and semantic keypoint ${}^{s}_{}\mathbf{z}_{k,l,j}$ measurements in the front-end. It also defines the error functions in \eqref{eq:problem} and their Jacobians needed for the back-end optimization. Finally, it presents the back-end optimization over the geometric landmarks $\calL$ and the object states $\calO$ for a given sensor sensor trajectory $\calX$.



\subsection{Keypoint and Object Tracking}
\label{sec:tracking}
Geometric keypoints ${}^g\bfz_{k,n}$ are detected in the camera images using the FAST detector \cite{FAST} and are tracked temporally using the Lucas-Kanade (LK) algorithm \cite{LK}. Keypoint-based tracking has lower accuracy but higher efficiency than descriptor-based methods, allowing our method to use a high frame-rate camera and process more keypoints. Outliers are eliminated by estimating the essential matrix between consecutive views and removing those keypoints that do not fit the estimated model. 
Assuming that the time between consecutive images is short, the relative orientation is obtained by integrating the gyroscope measurements ${}^i\bfomega_k$ as described in Sec.~\ref{sec:prediction_step} and only the unit translation vector is estimated using two-point RANSAC \cite{troiani20142}.

The YOLO detector \cite{bochkovskiy2020yolov4} is used to detect object classes ${}^{c}_{}\mathbf{z}_{k,j}$ and bounding-box lines ${}^{b}_{}\mathbf{z}_{k,l,j}$. Semantic keypoints ${}^{s}_{}\mathbf{z}_{k,l,j}$ are extracted within each bounding box using the StarMap stacked hourglass convolutional neural network~\cite{zhou2018starmap}. We augment the original StarMap network with dropout layers as shown in Fig.~\ref{fig:overview_all}(b). Several stochastic forward passes may be preformed using Monte Carlo dropout \cite{gal2015dropout} to obtain semantic keypoint covariances ${}^s\bfV$, illustrated in Fig.~\ref{fig:overview_all}(c). 

The bounding boxes ${}^{b}_{}\mathbf{z}_{k,l,j}$ are tracked temporally using the SORT algorithm~\cite{bewley2016simple}, which performs intersection over union (IoU) matching via the Hungarian algorithm. The semantic keypoints ${}^{s}_{}\mathbf{z}_{k,l,j}$ within each bounding box are tracked via a Kalman filter, which uses the LK algorithm for prediction and the StarMap keypoint detections for update.


\subsection{Landmark and Object Error Functions}
\label{sec:error_functions}

Next, we define the geometric-keypoint ${}^g\bfe_{k,m,n}$, semantic-keypoint ${}^s\bfe_{k,i,l,j}$, bounding-box ${}^b\bfe_{k,i,l,j}$, and regularization ${}^r\bfe(\bfo_i)$ error terms in \eqref{eq:problem} and derive their Jacobians. The error function arguments include the IMU, camera, and object poses, defined on the $SE(3)$ manifold, and, hence, particular care should be taken when obtaining the Jacobians. The error functions are linearized around estimates of the IMU-camera state $\hat{\bfx}_k$, geometric landmarks $\hat{\bfell}_m$, and objects $\hat{\bfo}_i$ using perturbations $\tilde{\bfx}_k$, $\tilde{\bfell}_m$, and $\tilde{\bfo}_i$:
\begin{equation}
\bfx_k = \hat{\bfx}_k \oplus \tilde{\bfx}_k, \qquad \bfell_m = \hat{\bfell}_m + \tilde{\bfell}_m, \qquad \bfo_i = \hat{\bfo}_i \oplus \tilde{\bfo}_i,
\end{equation}
where $\oplus$ emphasizes that some additions are over the $SE(3)$ manifold, defined as follows:
\begin{equation}
\label{eq:perturbations}
\scaleMathLine[0.91]{\begin{aligned}
{}_I\bfR &= {}_I\hat{\bfR} \exp\prl{{}_I\bftheta_\times} & {}_I\bfp&= {}_I\tilde{\bfp} + {}_I\hat{\bfp} & {}_I\bfv&= {}_I\tilde{\bfv} + {}_I\hat{\bfv}\\
{}_C\bfT &= {}_C\hat{\bfT} \exp\prl{{}_C\bfxi_\times} & \bfb_g &= \tilde{\bfb}_g + \hat{\bfb}_g & \bfb_a &= \tilde{\bfb}_a + \hat{\bfb}_a \\
{}_O\bfT &= {}_O\hat{\bfT} \exp\prl{{}_O\bfxi_\times} &\delta\bfs &= \delta\tilde{\bfs} + \delta\hat{\bfs} &\delta\bfu &= \delta\tilde{\bfu} + \delta\hat{\bfu},
\end{aligned}}
\end{equation}
where we use right perturbations ${}_I\bftheta_\times \in \mathfrak{so}(3)$, ${}_C\bfxi_\times \in \mathfrak{se}(3)$, and ${}_O\bfxi \in \mathfrak{se}(3)$ for the IMU orientation ${}_I\hat{\bfR}$, camera pose ${}_C\bfT$, and object pose ${}_O\hat{\bfT}$, respectively.




We define the geometric-keypoint error as the difference between the image projection of a geometric landmark $\bfell$ in camera frame ${}_{C}\bfT = \!{}_{I}\bfT\, {}^{I}_{C}\bfT$ and its associated keypoint observation ${}^{g}\mathbf{z}$:
\begin{equation}
\label{eq:gk-error}
{}^g\bfe\prl{\bfx, \bfell, {}^{g}\mathbf{z}} \triangleq \bfP \pi\prl{{}_C\bfT^{-1} \underline{\boldsymbol{\ell}}} - {}^{g}\mathbf{z},
\end{equation}
where $\bfP = \begin{bmatrix} \bfI_2 & \mathbf{0} \end{bmatrix} \in \mathbb{R}^{2 \times 4}$ is a projection matrix and $\pi(\underline{\bfs}) \triangleq \frac{1}{\bfs_3} \underline{\bfs} \in \mathbb{R}^4$ is the perspective projection function.

\begin{proposition}
\label{prop:gk-jacobians}
The Jacobians of ${}^g\bfe$ in \eqref{eq:gk-error} with respect to the IMU pose perturbation ${}_{I}\bfxi = [{}^{}_{I}{\bftheta}^\top\ {}^{}_{I}\tilde{\bfp}^\top]^\top$ and the landmark position perturbation $\tilde{\bfell}$, evaluated at estimates $\hat{\bfx}$, $\hat{\bfell}$, are:
\begin{equation}
\label{eq:gk-jacobians}
\begin{aligned}
\frac{\partial{}^g\bfe}{\partial{}_{I}\bfxi} &= -\bfP \frac{d\pi}{d\underline{\bfs}}\prl{{}_C\hat{\bfT}^{-1} \underline{\hat{\bfell}}} \brl{{}_C\hat{\bfT}^{-1} \underline{\hat{\bfell}}}^{\odot}
\frac{\partial {}_{C}\bfxi}{\partial {}_{I}\bfxi}
\in \mathbb{R}^{2 \times 6},\\
\frac{\partial{}^g\bfe}{\partial\tilde{\bfell}} &= \bfP \frac{d\pi}{d\underline{\bfs}}\prl{{}_C\hat{\bfT}^{-1} \underline{\hat{\bfell}}} {}_C\hat{\bfT}^{-1}\begin{bmatrix} \bfI_{3} \\ \mathbf{0}^\top \end{bmatrix} \;\;\in \mathbb{R}^{2 \times 3}, 
\end{aligned}
\end{equation}
where $[\cdot]^{\odot}$ is defined in \eqref{eq:odot}, $\frac{d\pi}{d\underline{\bfs}}(\underline{\bfs})$ is the Jacobian of $\pi(\underline{\bfs})$ and:
\begin{equation} \label{eq:dcxi-dixi}
\frac{\partial {}_{C}\bfxi}{\partial {}_{I}\bfxi} = 
\begin{bmatrix}
    - {}^{I}_{C}\bfR^\top {}^{I}_{C}\bfp_\times   
    & 
    {}^{}_{C}\hat{\bfR}^\top
    \\
    {}^{I}_{C}\bfR^\top
    & \mathbf{0}
\end{bmatrix} \in \bbR^{6 \times 6}
\end{equation}
The Jacobians with respect to other perturbations in \eqref{eq:perturbations} are $\mathbf{0}$. 
\end{proposition}

\begin{proof} 
See Sec. \ref{sec:prop1_proof}.
\end{proof}


The semantic-keypoint error is defined as the difference between the projection of a semantic landmark $\bfs_l+\delta\bfs_l$ from the object frame to the image plane, using instance pose ${}_{O}\bfT$ and camera pose ${}_{C}\bfT$, and its corresponding semantic-keypoint observation ${}^{s}\bfz$:
\begin{equation}
\label{eq:sk-error}
{}^s\bfe(\bfx, \bfo, {}^{s}\mathbf{z}) \triangleq \bfP \pi\prl{{}_C\bfT^{-1} {}_O\bfT\prl{\underline{\bfs}_l + \delta\underline{\bfs}_l} } - {}^{s}\mathbf{z}. 
\end{equation}
%

\begin{proposition}
\label{prop:sk-jacobians}
The Jacobians of ${}^s\bfe$ in \eqref{eq:sk-error} with respect to perturbations ${}_{I}\bfxi$, ${}_{O}\bfxi$, $\delta\tilde{\bfs}_l$, evaluated at estimates $\hat{\bfx}$, $\hat{\bfo}$, are:
\begin{gather}
\scaleMathLine{\begin{aligned}
\frac{\partial{}^s\bfe}{\partial{}_{I}\bfxi} &=
\frac{\partial{}^s\bfe}{\partial{}_{C}\bfxi} \frac{\partial {}_{C}\bfxi}{\partial {}_{I}\bfxi} \in \mathbb{R}^{2 \times 6},\\
\frac{\partial{}^s\bfe}{\partial{}_{C}\bfxi} &= -\bfP \frac{d\pi}{d\underline{\bfs}}\prl{{}_C\hat{\bfT}^{-1} {}_O\hat{\bfT}\prl{\underline{\bfs}_l + \delta\underline{\hat{\bfs}}_l} } \brl{{}_C\hat{\bfT}^{-1} {}_O\hat{\bfT}\prl{\underline{\bfs}_l + \delta\underline{\hat{\bfs}}_l}}^{\odot},\\ 
\frac{\partial{}^s\bfe}{\partial{}_{O}\bfxi} &= \bfP \frac{d\pi}{d\underline{\bfs}}\prl{{}_C\hat{\bfT}^{-1} {}_O\hat{\bfT}\prl{\underline{\bfs}_l + \delta\underline{\hat{\bfs}}_l} } {}_C\hat{\bfT}^{-1} {}_O\hat{\bfT} \brl{\prl{\underline{\bfs}_l + \delta\underline{\hat{\bfs}}_l}}^{\odot},\\
\frac{\partial{}^s\bfe}{\partial\delta\tilde{\bfs}_l} &= \bfP \frac{d\pi}{d\underline{\bfs}}\prl{{}_C\hat{\bfT}^{-1} {}_O\hat{\bfT}\prl{\underline{\bfs}_l + \delta\underline{\hat{\bfs}}_l} } {}_C\hat{\bfT}^{-1} {}_O\hat{\bfT} \begin{bmatrix} \bfI_{3} \\ \mathbf{0}^\top \end{bmatrix} \in \mathbb{R}^{2 \times 3}.
\end{aligned}}
\raisetag{20ex}
\end{gather}
The Jacobians with respect to other perturbations in \eqref{eq:perturbations} are $\mathbf{0}$. 
\end{proposition}

\begin{proof} 
See Sec. \ref{sec:prop2_proof}.
\end{proof}


We define the bounding-box error as the distance between the hyperplane $\underline{\bfb} = \brl{\bfb^\top\ -b_h}^\top \triangleq {}_O\bfT^\top  {}_C\bfT^{-\!\top} \bfP^\top {}^b\underline{\mathbf{z}}$ induced by projecting the bounding-box line ${}^b\mathbf{z}$ to the object frame and the closest hyperplane that is tangent to the quadric surface $\bfQ_{(\bfu+\delta\bfu)}^*$ of object $\bfo$:
\begin{equation} \label{eq:bb-error}
\scaleMathLine[0.89]{{}^b\bfe(\bfx, \bfo, {}^b\mathbf{z}) \triangleq  \frac{1}{\|\bfb\|} \!\prl{\!\sgn(b_h)\sqrt{\bfb^\top\! \diag(\bfu\!+\!\delta \bfu)^{2} \bfb  } - b_h \!}\!}
\end{equation}
where $\sgn(x) = \frac{\partial |x|}{\partial x}$ is the signum function.


\begin{proposition}
\label{prop:bb-jacobians}
The Jacobians of ${}^b\bfe$ in \eqref{eq:bb-error} with respect to perturbations ${}_{I}\bfxi$, ${}_{O}\bfxi$, $\delta\tilde{\bfu}$, evaluated at estimates $\hat{\bfx}$, $\hat{\bfo}$, are:
\begin{equation}
\begin{aligned}
\frac{\partial{}^b\bfe}{\partial {}_{I}\bfxi} &= 
\frac{\partial{}^b\bfe}{\partial \ubfb}
\frac{\partial \ubfb}{\partial {}_{C}\bfxi}
\frac{\partial {}_{C}\bfxi}{\partial {}_{I}\bfxi}
\in \bbR^{1 \times 6},
\\ 
\frac{\partial{}^b\bfe}{\partial{}_{O}\bfxi} &=
\frac{\partial{}^b\bfe}{\partial \ubfb}
\frac{\partial \ubfb}{\partial {}_{O}\bfxi}
\in \bbR^{1 \times 6},
\\ 
\frac{\partial{}^b\bfe}{\partial\delta\tilde{\bfu}} &= 
\frac{
\sgn(\hat{b}_h)(\bfu+\delta\hat{\bfu})^\top \diag(\hat{\bfb})^2 
}{ \|\hat{\bfb}\|\sqrt{\hat{\bfb}^\top \diag(\bfu + \delta \hat{\bfu})^2 \hat{\bfb}}}
\in \mathbb{R}^{1 \times 3},
\end{aligned}
\end{equation}
%
where $\underline{\hat{\bfb}} = {}_O\hat{\bfT}^\top  {}_C\hat{\bfT}^{-\!\top} \bfP^\top {}^b\underline{\mathbf{z}}$ and with $\hat{\bfU} \triangleq \diag(\bfu + \delta \hat{\bfu})$:
\begin{align}
\frac{\partial {}^b\bfe}{\partial \ubfb}
&= 
\frac{\sgn(\hat{b}_h)\hat{\bfb}^\top\hat{\bfU}^2 }{\|\hat{\bfb}\| \sqrt{\hat{\bfb}^\top \hat{\bfU}^2 \hat{\bfb}}}
\begin{bmatrix}
\left(
 \bfI_{3}
- \frac{\hat{\bfb}\hat{\bfb}^\top}{\|\hat{\bfb}\|^2}
\right)
&
\mathbf{0}
\end{bmatrix}
+
\frac{\begin{bmatrix}\hat{b}_h\hat{\bfb}^\top &  \|\hat{\bfb}\|^2\end{bmatrix} }{\|\hat{\bfb}\|^3},
\notag\\
\frac{\partial \underline{\bfb}}{\partial {}_C\bfxi}
&= 
-{}_O\hat{\bfT}^\top {}_C\hat{\bfT}^{-\top} \brl{\bfP^\top {}^b\underline{\mathbf{z}}}^{\circledcirc \top} 
\in \mathbb{R}^{4 \times 6},
\\ 
\frac{\partial \underline{\bfb}}{\partial {}_O\bfxi}
&=
\brl{{}_O\hat{\bfT}^\top {}_C\hat{\bfT}^{-\top} \bfP^\top {}^b\underline{\mathbf{z}}}^{\circledcirc \top} 
\in \mathbb{R}^{4 \times 6}.\notag
\end{align}
The Jacobians with respect to other perturbations in~\eqref{eq:perturbations} are $\mathbf{0}$. 
\end{proposition}

\begin{proof}
See Sec. \ref{sec:prop3_proof}.
\end{proof}

Finally, the object shape regularization error is defined as:
\begin{equation}
{}^r\bfe\prl{\bfo} \triangleq \begin{bmatrix} \delta\bfu^\top & \frac{1}{\sqrt{N_s}} \delta\bfs_1^\top & \cdots & \frac{1}{\sqrt{N_s}} \delta\bfs_{N_s}^\top \end{bmatrix}^\top, 
\end{equation}
whose Jacobians with respect to the perturbations $\delta\tilde{\bfu}$, $\delta\tilde{\bfs}_l$ are:
\begin{equation*}
\begin{aligned}
\frac{\partial {}^r\bfe}{\partial \delta\tilde{\bfu}} 
&=
\begin{bmatrix}
\bfI_3 & \mathbf{0}_{3N_s}  
\end{bmatrix}^\top
&\frac{\partial {}^r\bfe}{\partial \delta\tilde{\bfs}_l} 
&=
\begin{bmatrix}
\mathbf{0}_{3l} & \frac{1}{\sqrt{N_s}}\bfI_3 & \mathbf{0}_{3(N_s-l)}  
\end{bmatrix}^\top.
\end{aligned}
\end{equation*}


\subsection{Landmark and Object State Optimization}
We temporarily assume that the sensor trajectory $\calX$ is known. Given $\calX$, the optimization over $\calL$ and $\calO$ decouples into individual geometric landmark and object optimization problems. The error terms in the decoupled problems can be linearized around initial estimates $\hat{\bfell}_m$ and $\hat{\bfo}_i$, using the Jacobians in Propositions \ref{prop:gk-jacobians}, \ref{prop:sk-jacobians}, and \ref{prop:bb-jacobians}, leading to:
\begin{align}
\min_{\tilde{\bfell}_m} \;&{}^gw \sum_{k,n} \indicator_{k,m,n} \bigl\| {}^g\hat{\bfe}_{k,m,n} + \frac{\partial {}^g\hat{\bfe}_{k,m,n}}{\partial \tilde{\bfell}_m}\tilde{\bfell}_m \bigr\|_{{}^g\bfV}^2 \notag\\
\min_{\tilde{\bfo}_i} \; &{}^sw \sum_{k,l,j} \indicator_{k,i,j} \bigl\| {}^s\hat{\bfe}_{k,i,l,j} + \frac{\partial {}^s\hat{\bfe}_{k,i,l,j}}{\partial \tilde{\bfo}_i}\tilde{\bfo}_i \bigr\|_{{}^s\bfV}^2 \;+ \label{eq:object-lm}\\
\!{}^bw &\scaleMathLine[0.95]{\!\sum_{k,l,j} \!\indicator_{k,i,j} \bigl\| {}^b\hat{\bfe}_{k,i,l,j} \!+\! \frac{\partial {}^b\hat{\bfe}_{k,i,l,j}}{\partial \tilde{\bfo}_i}\tilde{\bfo}_i \bigr\|_{{}^b\bfV}^2 \!+ \!{}^rw \bigl\|{}^r\bfe(\hat{\bfo}_i) \!+\! \frac{\partial {}^r\bfe(\hat{\bfo}_i)}{\partial \tilde{\bfo}_i}\tilde{\bfo}_i \bigr\|^2}\notag
\end{align}
%
where $\frac{\partial {}^s\hat{\bfe}}{\partial \tilde{\bfo}} = \brl{\frac{\partial{}^s\bfe}{\partial{}_{O}\bfxi}\ \bf0\ \frac{\partial{}^s\bfe}{\partial\delta\tilde{\bfs}}}$ and $\frac{\partial {}^b\hat{\bfe}}{\partial \tilde{\bfo}} = \brl{\frac{\partial{}^b\bfe}{\partial{}_{O}\bfxi}\ \frac{\partial{}^b\bfe}{\partial\delta\tilde{\bfu}}\ \bf0}$. These unconstrained quadratic programs in $\tilde{\bfell}_m$ and $\tilde{\bfo}_i$ can be solved iteratively via the Levenberg-Marquardt algorithm \cite[Ch.4]{BarfootBook}, updating $\hat{\bfell}_m \gets \tilde{\bfell}_m + \hat{\bfell}_m$ and $\hat{\bfo}_i \gets \tilde{\bfo}_i \oplus \hat{\bfo}_i$ until convergence to a local minimum.



The geometric landmarks $\hat{\bfell}_m$ are initialized by solving ${}^g\hat{\bfe}_{k,m,n} = \mathbf{0}$ via the linear system of equations:
\begin{equation}
\bfP {}_C\hat{\bfT}_k^{-1} \underline{\hat{\bfell}}_m - \lambda_{k,n}{}^{g}\mathbf{z}_{k,n} = \mathbf{0}
\end{equation}
for all $k,m,n$ such that $\indicator_{k,m,n} = 1$, where the unknowns are $\hat{\bfell}_m$ and the keypoint depths $\lambda_{k,n}$. The deformations of an object instance $\hat{\bfo}_i$ are initialized as $\delta\hat{\bfs} = \mathbf{0}$ and $\delta\hat{\bfu} = \mathbf{0}$. The instance pose is determined from the system of equations consisting of semantic keypoint and bounding-box line residuals:
%
\begin{equation} \label{eq:initialization}
\begin{aligned}
\bfP{}_C\hat{\bfT}_k^{-1} {}_O\hat{\bfT} \underline{\bfs}_l - \lambda_{k,l,j} {}^{s}\mathbf{z}_{k,l,j} &= \mathbf{0}\\
{}^b\underline{\mathbf{z}}_{k,l,j}^\top \bfP {}_C\hat{\bfT}_k^{-1} {}_O\hat{\bfT} \bfQ_{\bfu}^*  {}_O\hat{\bfT}^\top  {}_C\hat{\bfT}_k^{-\!\top} \bfP^\top {}^b\underline{\mathbf{z}}_{k,l,j} &= 0
\end{aligned}
\end{equation}
for all $l$ and all $k,j$ such that $\indicator_{k,i,j} = 1$, where the unknowns are ${}_O\hat{\bfT}$ and the semantic keypoint depths $\lambda_{k,l,j}$.

The least squares problem for semantic keypoints is a generalization of the pose from $n$ point correspondences (PnP) problem \cite{PNP}. While this system may be solved using polynomial equations \cite{yang2020teaser}, we perform a more efficient initialization by defining $\bfzeta_l \triangleq  {}_O\hat{\bfR} \bfs_l + {}_O\hat{\bfp}$ and solving the first set of (now linear) equations in \eqref{eq:initialization} for $\bfzeta_l$ and $\lambda_{k,l,j}$.
We recover ${}_O\hat{\bfT}$ via the Kabsch algorithm \cite{Kabsch} between $\crl{\bfzeta_l}$ and $\crl{\bfs_l}$. 
This approach works well as long as there is a sufficient number of semantic keypoints ${}^{s}_{}\mathbf{z}_{k,l,j}$ (at least two per landmark across time for at least three semantic landmarks $\bfs_l$) associated with the object. If fewer semantic keypoints are available, ${}_O\hat{\bfT}$ can be recovered from the second equation in \eqref{eq:initialization}. Let $\bfQ^* \triangleq {}_O\hat{\bfT} \bfQ_{\bfu}^*  {}_O\hat{\bfT}^\top$, ${\bfpi}_{k,l,j} \triangleq {}^b\underline{\mathbf{z}}_{k,l,j}^\top \bfP {}_C\hat{\bfT}_k^{-1}$, then a linear system $\mathbf{M} \mathbf{w} = 0$ can be formed:
\begin{equation}
\begin{aligned}
  \mathbf{M} \triangleq
  \begin{bmatrix}
    \vdots\\
    \bfpi_{k,l,j}^\top \otimes \bfpi_{k,l,j}^\top\\
    \vdots
  \end{bmatrix}
  \quad 
  \mathbf{w} \triangleq \vectorize\prl{\bfQ^*}. 
\end{aligned}
\end{equation}  
for all $l$ and all $k,j$ such that $\indicator_{k,i,j} = 1$. The operator $\vectorize(\cdot)$ vectorizes a matrix by stacking its columns into a single column vector
and $\otimes$ is the Kronecker product. The solution can be found by:  
\begin{equation}
\label{eq:ellipsoid_lsq}
\hat{\mathbf{w}} = \argmin_{\bfw}\left\|\mathbf{M} \mathbf{w}\right\|_{2}^{2} \quad \text { s.t. } \quad\|\mathbf{w}\|_{2}^{2}=1,
\end{equation}
where the equality constraint avoids the trivial solution. The minimization in \eqref{eq:ellipsoid_lsq} can done via SVD of the $\mathbf{M}$ matrix \cite{rubino20173d}. The object pose ${}_O\hat{\bfT}$ can be recovered from $\hat{\mathbf{w}} =  \vectorize(\hat{\bfQ}^*)$ by relating the estimated ellipsoid $\hat{\mathbf{Q}}^*$ in global coordinates to the ellipsoid $\bfQ^*_{\bfu}$ in the object frame:
\begin{equation*}
\begin{aligned}
\hat{\bfQ}^*\! =
{}_O\hat{\bfT} \bfQ_{\bfu}^* {}_O\hat{\bfT}^{\top}\!\!=
\begin{bmatrix} 
  {}_O\hat{\bfR} \bfU\bfU^\top {}_O\hat{\bfR}^\top -  {}_O\hat{\bfp} {}_O\hat{\bfp}^\top & - {}_O\hat{\bfp} \\ -{}_O\hat{\bfp}^\top & -1
\end{bmatrix}.
\end{aligned}
\end{equation*}
The translation ${}_O\hat{\bfp}$ can be recovered from the last column of $\hat{\bfQ}^*$.
To recover the rotation, note that $\bfA \triangleq \bfP\hat{\mathbf{Q}}^*\bfP^\top  + {}_O\hat{\bfp} {}_O\hat{\bfp}^\top = {}_O\hat{\bfR}\bfU\bfU^\top{}_O\hat{\bfR}^\top$ is a positive semidefinite matrix. Let its eigen-decomposition be $\bfA = \bfV\bfY\bfV^\top$, where $\bfY$ is a diagonal matrix containing the eigenvalues of $\bfA$. Since $\bfU\bfU^\top$ is diagonal, it follows that ${}_O\hat{\bfR} = \bfV$.


%% file: tex/OrcVIO.tex

\section{The OrcVIO Algorithm}
\label{sec:orcvio}

We return to the problem of joint IMU-camera, geometric-landmark, and object optimization and describe the OrcVIO algorithm. The IMU-camera state $\bfx_k$ is tracked using an extended Kalman filter with mean $\hat{\bfx}_k$ and covariance $\bfSigma_k$. Prediction of the mean $\hat{\bfx}_{k+1}^{p}$ and covariance $\bfSigma_{k+1}^{p}$ is performed using the inertial measurements ${}^i\bfz_k$. When a geometric-keypoint or object track is lost, iterative optimization is performed over $\hat{\bfell}_m$ and $\hat{\bfo}_i$ as discussed in Sec.~\ref{sec:reconstruction}. The optimized geometric landmark and object estimates are used to update the IMU-camera mean $\hat{\bfx}_{k+1}^{p}$ and covariance $\bfSigma_{k+1}^{p}$. OrcVIO is an extension of the MSCKF \cite{msckf}, which performs closed-form prediction and integrates object states in the update.



\subsection{Prediction Step}
\label{sec:prediction_step}
The continuous-time IMU dynamics are~\cite{quatekf}:
\begin{align} \label{eq:imu-dynamics}
{}_I\dot{\bfR} &= {}_I\bfR \prl{{}^i\bfomega - \bfb_g - {\bfn}_{\boldsymbol{\omega}}}_{\times}, \quad \dot{\bfb}_g = \bfn_g,  \;\;\quad \dot{\bfb}_a = \bfn_a, \notag\\
{}_I\dot{\bfv} &= {}_I\bfR \prl{{}^i\bfa - \bfb_a - \bfn_{\bfa}} + \bfg, \qquad {}_I\dot{\bfp} = {}_I\bfv,
\end{align}
where ${\bfn}_{\boldsymbol{\omega}}$, $\bfn_{\bfa}$, $\bfn_g$, $\bfn_a \in \mathbb{R}^3$ are white Gaussian noise terms associated with the angular velocity, linear acceleration, gyroscope bias, and accelerometer bias, respectively. The power spectral densities $\sigma_{\bfomega}^2\bfI_3$ $[rad^2/s]$, $\sigma_{\bfa}^2\bfI_3$ $[m^2/s^3]$, $\sigma_{g}^2\bfI_3$ $[rad^2/s^3]$, $\sigma_{a}^2\bfI_3$ $[m^2/s^5]$ can be obtained from the IMU datasheet or experimental data \cite[Appendix E]{sola2012quaternion}. Using the perturbations in \eqref{eq:perturbations}, we can split \eqref{eq:imu-dynamics} into deterministic nominal dynamics:
\begin{align}
\label{eq:imu-error-dynamics-nominal}
{}_I\dot{\hat{\bfR}} &= {}_I\hat{\bfR} \prl{{}^i\bfomega - \hat{\bfb}_g}_{\times}, \qquad \dot{\hat{\bfb}}_g = \mathbf{0},  \quad\qquad \dot{\hat{\bfb}}_a = \mathbf{0}, \notag\\
{}_I\dot{\hat{\bfv}} &= {}_I\hat{\bfR} \prl{{}^i\bfa - \hat{\bfb}_a} + \bfg,\qquad {}_I\dot{\hat{\bfp}} = {}_I\hat{\bfv},
\end{align}
and stochastic error dynamics:
\begin{equation}
\label{eq:imu-error-dynamics-stochastic}   
\begin{aligned}
{}_I\dot{\bftheta} &= -\prl{{}^i\bfomega - \hat{\bfb}_g}_\times {}_I\bftheta - \prl{\tilde{\bfb}_g + {\bfn}_{\boldsymbol{\omega}}},\\
{}_I\dot{\tilde{\bfv}} &= - {}_I\hat{\bfR}\prl{{}^i\bfa - \hat{\bfb}_a}_\times{}_I\bftheta - {}_I\hat{\bfR}\prl{\tilde{\bfb}_a + {\bfn}_{\bfa}},\\
{}_I\dot{\tilde{\bfp}} &= {}_I\tilde{\bfv}, \qquad \dot{\tilde{\bfb}}_g = \bfn_g, \qquad \dot{\tilde{\bfb}}_a = \bfn_a.
\end{aligned}
\end{equation}
%
Given time discretization $\tau_k$ and assuming ${}^i\bfomega(t)$ and ${}^i\bfa(t)$ remain constant over the interval $t \in [t_k,t_{k}+\tau_k)$ with values ${}^i\bfz_k = ({}^i\bfomega_k, {}^i\bfa_k)$, we can compute the predicted IMU-camera state mean and covariance from \eqref{eq:imu-error-dynamics-nominal} and \eqref{eq:imu-error-dynamics-stochastic}, respectively. Let $\hat{\bfx}_k$ and $\bfSigma_k$ be the prior mean and covariance.


\begin{proposition}
\label{prop:closed_form_mean_prop}
The nominal dynamics \eqref{eq:imu-error-dynamics-nominal} can be integrated in \emph{closed-form} to obtain the predicted mean $\hat{\bfx}_{k+1}^{p}$:
\begin{align}
\label{eq:imu-integral}
&{}_I\hat{\bfR}_{k+1}^{p} \!=\! {}_I\hat{\bfR}_{k} \exp\prl{ \tau_k \bigl({}^i\bfomega_k \!- \hat{\bfb}_{g,k}\bigr)_{\times}},\notag\\
&\scaleMathLine{{}_I\hat{\bfv}_{k+1}^{p} \!=\!{}_I\hat{\bfv}_k + \bfg \tau_k + {}_I\hat{\bfR}_{k} \bfJ_L\!\prl{\tau_k \bigl({}^i\bfomega_k \!- \hat{\bfb}_{g,k}\bigr)} \bigl({}^i\bfa_k \!- \hat{\bfb}_{a,k}\bigr) \tau_k},\notag\\
&\scaleMathLine{{}_I\hat{\bfp}_{k+1}^{p} \!=\!{}_I\hat{\bfp}_k \!+\! {}_I\hat{\bfv}_k\tau_k \!+\! \bfg\frac{\tau_k^2}{2} \!+\! {}_I\hat{\bfR}_{k}\bfH_L\!\prl{\tau_k \bigl({}^i\bfomega_k \!-\!\hat{\bfb}_{g,k}\bigr)}\bigl({}^i\bfa_k \!-\! \hat{\bfb}_{a,k}\bigr) \tau_k^2}, \notag\\
&\hat{\bfb}_{g,k+1}^{p} = \hat{\bfb}_{g,k}, \qquad\qquad \hat{\bfb}_{a,k+1}^{p} = \hat{\bfb}_{a,k},\\
&\scaleMathLine{{}_I\hat{\bfT}_{k}^p = \begin{bmatrix} {}_I\hat{\bfR}_k & {}_I\hat{\bfp}_k \\ \mathbf{0}^\top & 1 \end{bmatrix}, {}_I\hat{\bfT}_{k-1}^p = {}_I\hat{\bfT}_{k-1}, \ldots, {}_I\hat{\bfT}_{k-W+1}^p = {}_I\hat{\bfT}_{k-W+1},}\notag
\end{align}
where $\bfJ_L\prl{\bfomega} \triangleq {\bfI}_3 + \frac{\bfomega_{\times}}{2!} + \frac{\bfomega_{\times}^2}{3!} + \ldots$ is the left Jacobian of $SO(3)$ and $\bfH_L\prl{\bfomega} \triangleq \frac{{\bfI}_3}{2!} +  \frac{\bfomega_{\times}}{3!} + \frac{\bfomega_{\times}^2}{4!} + \ldots$. 
Both $\bfJ_L\prl{\bfomega}$ and $\bfH_L\prl{\bfomega}$ admit closed-form (Rodrigues) expressions:
\begin{gather} \label{eq:JH-rodrigues}
\scaleMathLine{
\begin{aligned}
\bfJ_L(\bfomega) &= {\bfI}_3 + \frac{1-\cos\|\bfomega\|}{\|\bfomega\|^2}\bfomega_{\times} + \frac{\|\bfomega\|-\sin\|\bfomega\|}{\|\bfomega\|^3}\bfomega_{\times}^2\\
\bfH_L(\bfomega) &= \frac{1}{2}{\bfI}_3 + \frac{\|\bfomega\|-\sin\|\bfomega\|}{\|\bfomega\|^3}\bfomega_{\times} + \frac{2(\cos\|\bfomega\|-1)+\|\bfomega\|^2}{2\|\bfomega\|^4}\bfomega_{\times}^2.
\end{aligned}}
\raisetag{10ex}
\end{gather}
\end{proposition}

\begin{proof}
See Sec. \ref{sec:imu-integral-proof}. 
\end{proof}

To compute the predicted covariance $\bfSigma_{k+1}^p$, we need to integrate the error dynamics in \eqref{eq:imu-error-dynamics-stochastic}. The IMU error state ${}^{}_{I}\tilde{\bfx}(t)$ satisfies a linear time-variant (LTV) stochastic differential equation (SDE) for $t \in [0,\tau_k)$:
\begin{equation}
\label{eq:smsckf_eq2}
{}^{}_{I}\dot{\tilde{\mathbf{x}}}
=
\bfF(t) {}^{}_{I}\tilde{\mathbf{x}} + {}^{}_{I}\mathbf{n}, \qquad {}^{}_{I}\tilde{\mathbf{x}}(0) \sim \mathcal{N}(\mathbf{0},{}_{I}\bfSigma_k)
\end{equation}
where ${}_{I}\bfSigma_k$ is the top-left $15 \times 15$ block of $\bfSigma_k$ corresponding to the IMU state, ${}^{}_{I}\mathbf{n}$ is white Gaussian noise with power spectral density $\bfQ = \diag(\sigma_{\bfomega}^2 {\bfI}_3, \sigma_{\bfa}^2 {\bfI}_3, \mathbf{0}_3, \sigma_{g}^2 {\bfI}_3, \sigma_{a}^2 {\bfI}_3)$ and ${\bfF}(t)$ is:
\begin{equation*}
{\bfF}(t) = 
\begin{bmatrix}
-\prl{{}^i\bfomega_k - \hat{\bfb}_{g,k}}_\times & 
{\bf0} & 
{\bf0} &
-{\bfI}_3 & 
{\bf0} 
\\
-{}_{I}\hat{\bfR}(t) \prl{{}^i\bfa_k - \hat{\bfb}_{a,k}}_\times& 
{\bf0} &
{\bf0} & 
{\bf0} &
-{}_{I}\hat{\bfR}(t)
\\
{\bf0} &
{\bfI}_3 & 
{\bf0} & 
{\bf0} & 
{\bf0}
\\
{\bf0} & {\bf0} & {\bf0} & {\bf0} & {\bf0} 
\\
{\bf0} & {\bf0} & {\bf0} & {\bf0} & {\bf0}
\end{bmatrix},
\end{equation*}
where ${}_{I}\hat{\bfR}(t) = {}_{I}\hat{\bfR}_k \exp\prl{ t \bigl({}^i\bfomega_k \!- \hat{\bfb}_{g,k}\bigr)_{\times}}$. Since $\bbE\brl{{}^{}_{I}\tilde{\bfx}(t)}$ remains zero over $[0,\tau_k)$, the covariance of ${}^{}_{I}\tilde{\mathbf{x}}(\tau_k)$ is:
\begin{align} \label{eq:prop_closed_form}
    {}^{}_{I}\bfSigma_{k+1}^p \!&= \bbE\brl{{}^{}_{I}\tilde{\bfx}(\tau_k){}^{}_{I}\tilde{\bfx}(\tau_k)^\top} \\
    &= \bfPhi(\tau_k,0) {}^{}_{I}\bfSigma_{k}\bfPhi(\tau_k,0)^\top\! + \int_{0}^{\tau_k} \bfPhi(\tau_k,s) \bfQ \bfPhi(\tau_k,s)^\top\! ds \notag
\end{align}
where $\bfPhi(t,s)$ is the transition matrix\footnote{Note that ${}_{I}\hat{\bfR}(t)$ in $\bfF(t)$ is time varying. Time-invariant approximations of the transition matrix are presented in \cite[App. B, E]{sola2012quaternion} using right-perturbation error dynamics and in \cite{orcvio} using left-perturbation error dynamics.} of \eqref{eq:smsckf_eq2}.

\input{tex/CovProp.tex}

Using \eqref{eq:prop_closed_form} and Proposition~\ref{prop:closed_form_cov_prop}, we can approximate the discrete-time covariance propagation for the IMU state as ${}_{I}\boldsymbol{\Sigma}_{k+1}^p \approx \bfPhi(\tau_k,0) \prl{ {}_{I}\bfSigma_k + \tau_k \bfQ} \bfPhi(\tau_k,0)^\top$. The full covariance propagation, accounting for the addition of the latest IMU pose ${}_I\hat{\bfT}_{k}$ to the sliding window and the dropping of the oldest IMU pose ${}_I\hat{\bfT}_{k-W}$, is:
\begin{align}
\bfSigma_{k+1}^p &\approx \bfA_k\prl{\bfSigma_{k} + \tau_k \diag(\bfQ,\mathbf{0}_{6W})} \bfA_k^\top,\\
\bfA_k &\triangleq \begin{bmatrix} \bfPhi(\tau_k,0) & \mathbf{0} & \mathbf{0}\\ \bfB & \mathbf{0} & \bf0 \\ \bf0 & \bfI_{6(W-1)} & \mathbf{0} \end{bmatrix} \quad 
\bfB \triangleq \begin{bmatrix}
{\bfI}_3 & \bf0 & \bf0 & \mathbf{0}\\
\bf0 & \bf0 & \mathbf{I}_3 & \mathbf{0}
\end{bmatrix}. \notag
\end{align}

\begin{figure*}[t]
  \centering
  \begin{minipage}{.27\linewidth}
    \includegraphics[width=\linewidth]{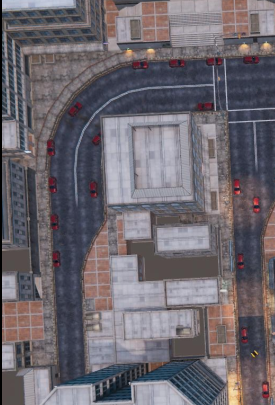}
    \caption{Urban scene simulation in the Unity game engine with car, door, and barrier object instances.}
    \label{fig:unity_scene}
  \end{minipage}%
  \hfill%
  \begin{minipage}{.71\linewidth}
    \includegraphics[width=\linewidth]{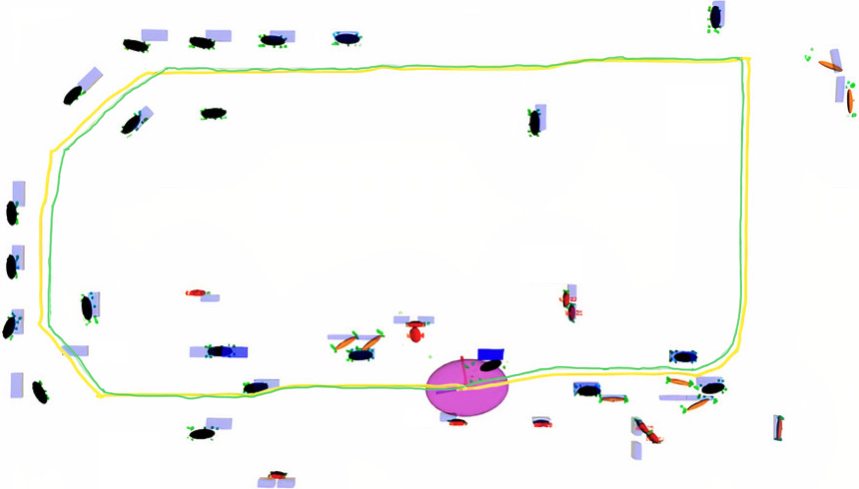}
    \caption{Object map reconstructed by OrcVIO on the Unity dataset. The figure shows the ground-truth trajectory (green curve), estimated trajectory (yellow curve), the covariance of the last pose (purple ellipsoid), the ground-truth objects (blue cuboids), estimated object ellipsoids (cars are black, doors are red, barriers are orange), as well as the reconstructed semantic landmarks (green dots).}
    \label{fig:unity_object_map}
  \end{minipage}
\end{figure*}

\subsection{Update Step}



We perform an update to $\hat{\bfx}_{k+1}^p$ and $\bfSigma_{k+1}^p$ without storing the geometric landmarks $\hat{\bfell}_m$ or object instances $\hat{\bfi}_i$ in the filter state using the null-space projection idea of \cite{msckf}. Let $\hat{\bfy}_i$ denote an estimate (from Sec.~\ref{sec:reconstruction}) of a geometric landmark or object instance whose track gets lost at time $t_{k+1}$. The geometric and semantic error functions, linearized using perturbations ${}_{C}\bfxi_{k}^p$, $\tilde{\bfy}_i$ around estimates ${}_C\hat{\bfT}_{k}^p$, $\hat{\bfy}_i$ are of the form:
\begin{equation}
\bfe_{k,i} \approx \hat{\bfe}_{k,i} + \frac{\partial\hat{\bfe}_{k,i}}{\partial{}_{C}\bfxi_{k}^p} {}_{C}\bfxi_{k}^p + \frac{\partial\hat{\bfe}_{k,i}}{\partial\tilde{\bfy}_i} \tilde{\bfy}_i + \bfn_{k,i}
\end{equation}
where $\bfn_{k,i}$ is the associated noise term with covariance $\bfV_{k,i}$. Stacking the errors for all camera poses in $\hat{\bfx}_{k+1}^p$ associated with $i$, leads to:
\begin{equation}
\label{eq:stacked-residual}
\bfe_{i} \approx \hat{\bfe}_{i} + \frac{\partial\hat{\bfe}_{i}}{\partial\tilde{\bfx}_{k+1}^p} \tilde{\bfx}_{k+1}^p + \frac{\partial\hat{\bfe}_{i}}{\partial\tilde{\bfy}_i} \tilde{\bfy}_i + \bfn_{i}.
\end{equation}
The perturbations $\tilde{\bfy}_i$ can be eliminated by left-multiplication of the errors in~\eqref{eq:stacked-residual} with unitary matrices $\bfN_i$ whose columns form the basis of the left nullspaces of $\frac{\partial \hat{\bfe}_i}{\partial\tilde{\bfy}_i}$:
\begin{equation}
\bfN_i^\top\bfe_i \approx \bfN_i^\top\hat{\bfe}_i + \bfN_i^\top\frac{\partial\hat{\bfe}_i}{\partial\tilde{\bfx}_{k+1}^p} \tilde{\bfx}_{k+1}^p + \bfN_i^\top\bfn_i. 
\end{equation}

In Sec. \ref{sec:zero_velocity_update}, we also introduce a zero-velocity residual to account for complete stops of the system, which is common for autonomous ground and some aerial robots.


Finally, let $\hat{\bfe}$, $\bfJ$, $\bfV$ be the stacked errors, Jacobians, and noise covariances (after null-space projection) across all geometric landmarks and object instances, whose tracks are lost at $t+1$. The updated IMU-camera mean and covariance are:
\begin{align}
\bfK &= \bfSigma_{k+1}^p \bfJ^\top \prl{\bfJ\bfSigma_{k+1}^p\bfJ^\top+\bfV}^{-1}
\notag\\
\hat{\bfx}_{k+1} &= \hat{\bfx}_{k+1}^p \oplus \prl{-\bfK \hat{\bfe}} \label{eq:ekf_update}\\
\bfSigma_{k+1} &= (\bfI - \bfK \bfJ) \bfSigma_{k+1}^p(\bfI- \bfK\bfJ)^\top + \bfK \bfV \bfK^\top.\notag
\end{align}
Note that the dimension of $\bfJ$ can be reduced in the computation above via QR factorization as described in \cite{msckf}.

%% file: tex/CovProp.tex

\begin{proposition}
\label{prop:closed_form_cov_prop}

The LTV SDE in \eqref{eq:smsckf_eq2} has a \emph{closed-form} transition matrix: 
%
\begin{equation*}
\bfPhi(t,0) = 
\begin{aligned}
\begin{bmatrix}
\exp\prl{-t\bfomega_{\times}} & 
{\bf0} & 
{\bf0} &
-t\bfJ_L\prl{-t\bfomega} & 
{\bf0} 
\\
{\bfPhi}_{\bfv\bftheta}(t) & 
{\bfI}_3 & 
{\bf0} & 
{\bfPhi}_{\bfv\bfomega}(t) &
{\bfPhi}_{\bfv\bfa}(t)
\\
{\bfPhi}_{\bfp\bftheta}(t) &
t {\bfI}_3 & 
{\bfI}_3 & 
{\bfPhi}_{\bfp\bfomega}(t) & 
{\bfPhi}_{\bfp\bfa}(t)
\\
{\bf0} & {\bf0} & {\bf0} & \mathbf{I}_3 & {\bf0} 
\\
{\bf0} & {\bf0} & {\bf0} & {\bf0} & \mathbf{I}_3 
\end{bmatrix}
\end{aligned}
\end{equation*}
where $\bfw = {}^i\bfomega_k - \hat{\bfb}_{g,k}$, $\bfa = {}^i\bfa_k - \hat{\bfb}_{a,k}$ and the blocks are:
\begin{gather} \label{eq:Phi-blocks}
\scaleMathLine{\begin{aligned}
&{\bfPhi}_{\bfv\bftheta}(t) = -t {}_{I}\hat{\bfR}_k\left[\bfJ_L\prl{t\bfomega} \bfa\right]_\times 
\\
&{\bfPhi}_{\bfv\bfomega}(t) = {}_{I}\hat{\bfR}_k \Delta(t) \frac{\bfa_{\times}}{\|\bfomega\|^2}\prl{\bfI_3 + \frac{\bfomega_{\times}^2}{||\bfomega||^2}}
\\
&\quad + t{}_{I}\hat{\bfR}_k \prl{\frac{\bfomega \bfa^\top}{\|\bfomega\|^2} \prl{\bfJ_L(-t\bfomega) - \bfI_3} +  \frac{\bfa^\top\bfomega }{\|\bfomega\|^2} \prl{\bfJ_L(t\bfomega) - \bfI_3}}
\\
&{\bfPhi}_{\bfv\bfa}(t) = -t {}_{I}\hat{\bfR}_k \bfJ_L\prl{t\bfomega}
\\ 
&{\bfPhi}_{\bfp\bftheta}(t) = -t^2 {}_{I}\hat{\bfR}_k \brl{\bfH_L\prl{t\bfomega} \bfa}_\times 
\\ 
&{\bfPhi}_{\bfp\bfomega}(t) = {}_{I}\hat{\bfR}_k 
\prl{\!t\bfJ_L\prl{t\bfomega}\! - \!\frac{\bfomega_{\times}}{\|\bfomega\|^2}\Delta(t)\! - t\bfI_3\!}
\frac{\bfa_{\times}}{\|\bfomega\|^2}\prl{\!\bfI_3 + \frac{\bfomega_{\times}^2}{\|\bfomega\|^2}\!} \\ 
&\quad + \frac{t^2}{2}{}_{I}\hat{\bfR}_k \prl{\!\frac{\bfomega \bfa^\top}{\|\bfomega\|^2} \prl{2\bfH_L(-t\bfomega)- \bfI_3} + \frac{\bfa^\top\bfomega }{\|\bfomega\|^2} \prl{2\bfH_L(t\bfomega) - \bfI_3} \!}
\\ 
&{\bfPhi}_{\bfp\bfa}(t) = -t^2 {}_{I}\hat{\bfR}_k \bfH_L\prl{t\bfomega}
\end{aligned}}
\raisetag{20ex}
\end{gather}
where $\Delta(t) = \prl{\exp\prl{t\bfomega_{\times}}\prl{\bfI_3 - t\bfomega_{\times}} - \bfI_3}$.
\end{proposition}

\begin{proof}
See Sec.~\ref{sec:closed_form_cov_prop}. 
\end{proof}


%% file: tex/Evaluation.tex
\section{Evaluation}
\label{sec:evaluation}

This section presents results from large-scale evaluation of OrcVIO using indoor and outdoor simulated and real data. We present experiments in photo-realistic Unity simulation (Fig.~\ref{fig:unity_scene}), on raw and odometry data sequences from the KITTI dataset \cite{geiger2012we}, and on real data collected on UCSD's campus.


\subsection{Metrics}

We use two standard metrics for quantitative VIO evaluation: position Root Mean Square Error (RMSE) \cite{sturm2012benchmark} (also referred to as position ATE \cite{Zhang18iros}) and KITTI's translation error (TE) metric \cite{geiger2012we}. Let $\operatorname{Trans}(\bfT)$ return the position component of a pose $\bfT \in SE(3)$. Let ${}_{I}\mathbf{T}_k$ be the ground-truth pose trajectory and ${}_{I}\hat{\mathbf{T}}_k'$ be the estimated pose trajectory. To measure error, the estimated trajectory is first aligned to the initial frame of the ground-truth trajectory via ${}_{I}\hat{\mathbf{T}}_k = {}_{I}\mathbf{T}_0{}_{I}\hat{\mathbf{T}}_0^{-1}  {}_{I}\hat{\mathbf{T}}_k'$. After alignment, RMSE (m) and TE (\%) are measured as:
\begin{align}
    &\text{RMSE} \triangleq \prl{\frac{1}{K} 
\sum_{k=0}^{K-1} \left\| \operatorname{Trans}\prl{{}_{I}\mathbf{T}_k^{-1}  {}_{I}\hat{\mathbf{T}}_k} \right\|_2^2 }^{1/2},\\
    &\text{TE} \triangleq \frac{1}{|\mathcal{F}|} \sum_{(i, j) \in \mathcal{F}} \frac{\left\|\operatorname{Trans}\biggl(\prl{{}_I\hat{\bfT}_{j}^{-1}{}_I\hat{\bfT}_{i}}^{-1}\prl{{}_I\bfT_{j}^{-1}{}_I\bfT_{i}}\biggr)\right\|_2}{\operatorname{length}(i, j)},\notag
\end{align}
where $\mathcal{F}$ is a set of frames with fixed distances $\operatorname{length}(i, j)$ over the values  $\crl{100,...,800}\, m$ \cite{geiger2012we}.


The object estimates are evaluated using 3D Intersection over Union (IoU). A 3D bounding box $\hat{b}_i$ is obtained from each estimated object $\hat{\bfo}_i$, and IoU is defined as the ratio of the intersection volume over the union volume with respect to the bounding box $b_i$ of the closest ground truth object:
\begin{equation}
    \text{IoU}(\hat{b}_i,b_i) \triangleq \sum\limits_{i}\frac{\text{Volume of Intersection}(\hat{b}_i,b_i)}{\text{Volume of Union}(\hat{b}_i,b_i)}.
\end{equation}
To understand the distribution of the object orientation and translation errors, we define an estimate as \textit{true positive} if the closest ground-truth object pose is within a specific rotation or translation error threshold. Specifically, a rotation error of $\alpha^\circ$ means $\|\log({}_{O}\hat{\bfR}^{\top}{}_{O}\mathbf{R})^\vee\|_2 \leq \alpha^\circ$, and translation error of $\beta\ m$ means $\| {}_{O}\mathbf{p} - {}_{O}\hat{\mathbf{p}}\|_2 \leq \beta\ m$. We define \textit{precision} as the fraction of true positives over all estimated objects and \textit{recall} is the fraction of true positives over all ground-truth objects.

\subsection{Unity Dataset Results}

OrcVIO is evaluated in a Unity simulation containing 50 car, road barrier, and door object instances, shown in Fig. \ref{fig:unity_scene}. A ROS bridge between Unity and Gazebo is used to simulate a quadrotor robot, navigating in the environment and providing IMU and camera measurements. The object map reconstructed by OrcVIO is shown in Fig. \ref{fig:unity_object_map}. Additional qualitative results and a video demo are available in the \hyperref[sec:supp]{\textcolor{Blue}{Supplementary Material}}. The estimated objects are generally quite close to the ground-truth ones. The object poses near the starting position are less accurate due to insufficient motion parallax, since the quadrotor performs a pure rotation in the beginning. The trajectory RMSE and TE are $0.97\,m$ and $0.40\%$, respectively. The odometry drift is mainly due to pure rotation maneuvers at planned path corners executed by the quadrotor controller. 

The 3D IoU of the object estimates is $0.49$. The Precision and Recall of the object reconstruction is shown in Table \ref{tab:pr_unity}. Despite that doors and barriers have a thin structure, causing even small pose estimation drift to reduce the overlap with the ground-truth object instances, OrcVIO is able to produce an accurate object map with good 3D IoU.




\begin{table*}[t]
  \centering
 \caption {Precision-Recall Evaluation on the Unity Dataset} \label{tab:pr_unity}
	\begin{tabular}{l rr rr rr}
    \toprule
    Translation error $\rightarrow$ & \multicolumn{2}{c}{$\leq 0.5\text{ m}$} & \multicolumn{2}{c}{$\leq 1.0\text{ m}$} & \multicolumn{2}{c}{$\leq 1.5\text{ m}$} \\
    \cmidrule(lr){2-3}\cmidrule(lr){4-5}\cmidrule(lr){6-7}
    Rotation error & Precision & Recall & Precision & Recall  & Precision & Recall \\
    \midrule
    \multirow{1}{*}{$\leq 30^{\circ}$} 
      & 0.05 & 0.05 & 0.18 & 0.20 & 0.22 & 0.25  \\
    \midrule[0.5\lightrulewidth]
    \multirow{1}{*}{$\leq 45^{\circ}$} 
      & 0.09 & 0.10 & 0.27 & 0.30 & 0.45 & 0.50  \\
    \midrule[0.5\lightrulewidth]
    \multirow{1}{*}{$-$} 
      & 0.14 & 0.15 & 0.41 & 0.45 & 0.64 & 0.70  \\
    \bottomrule
  \end{tabular}
\end{table*}


\begin{table*}[t]
  \centering
 \caption {Object Detection and Pose Estimation on the KITTI Object Sequences} 
	\begin{tabular}{l l rrrrr rrrrr }
\toprule
Metric & KITTI Sequence $\rightarrow$ &  22 & 23 & 36 & 39 & 61 & 64 & 95 & 96 & 117 & Mean \\
\midrule
\multirow{3}{*}{3D IoU} & SingleView \cite{mousavian20173d} &  0.52 & 0.32 & 0.50 & 0.54 & \bf{0.54} & 0.43 & 0.40 & 0.26 & 0.25 & 0.42 \\
& CubeSLAM \cite{cubeslam} &  \bf{0.58} & 0.35 & \bf{0.54} & \bf{0.59} & 0.50 & \bf{0.48} & \bf{0.52} & 0.29 & \bf{0.35} & \bf{0.47} \\
& OrcVIO & 0.51 & \bf{0.55} & 0.53 & 0.55 & 0.53 & 0.46 & 0.29 & 0.31 & 0.23 & 0.44  \\
\midrule[0.5\lightrulewidth]
\multirow{2}{*}{Trans. error (\%)} & CubeSLAM \cite{cubeslam} & \bf{1.68} & 1.72 & \bf{2.93} & 1.61 & 1.24 & \bf{0.93} & \bf{1.49}  & 1.81 & 2.21 & 1.74 \\
& OrcVIO & \bf{1.68} & \bf{1.50} & 2.95 & \bf{1.44} & \bf{1.22} & 1.02 & \bf{1.49} & \bf{1.59} & \bf{1.92} & \bf{1.65} \\
\bottomrule
  \end{tabular}
\label{tab:3d_iou}
\end{table*}

\begin{table*}[t]
  \centering
 \caption {Precision-Recall Evaluation on the KITTI Object Sequences} \label{tab:pr_kitti}
	\begin{tabular}{l l rr rr rr}
    \toprule
    & Translation error $\rightarrow$ & \multicolumn{2}{c}{$\leq 0.5\text{ m}$} & \multicolumn{2}{c}{$\leq 1.0\text{ m}$} & \multicolumn{2}{c}{$\leq 1.5\text{ m}$} \\
    \cmidrule(lr){3-4}\cmidrule(lr){5-6}\cmidrule(lr){7-8}
    Rotation error & Method & Precision & Recall & Precision & Recall  & Precision & Recall \\
    \midrule
    \multirow{3}{*}{$\leq 30^{\circ}$} & SubCNN~\cite{xiang2017subcategory} & 0.10 & 0.07 & 0.26 & 0.17 & 0.38 & 0.26 \\
    & VIS-FNL~\cite{dongFS17} & \bf{0.14} & 0.10 & \bf{0.34} & 0.24 & \bf{0.49} & \bf{0.35} \\
     & OrcVIO & \bf{0.14} & \bf{0.16} & 0.23 & \bf{0.25} & 0.26 & 0.29  \\
    \midrule[0.5\lightrulewidth]
\multirow{3}{*}{$\leq 45^{\circ}$} & SubCNN~\cite{xiang2017subcategory} & 0.10 & 0.07 & 0.26 & 0.17 & 0.38 & 0.26 \\
    & VIS-FNL~\cite{dongFS17} & 0.15 & 0.11 & \bf{0.35} & 0.25 & \bf{0.50} & 0.36 \\
     & OrcVIO & \bf{0.19} & \bf{0.22} & 0.31 & \bf{0.35} & 0.36 & \bf{0.40}  \\
    \midrule[0.5\lightrulewidth]
    \multirow{3}{*}{$-$} & SubCNN~\cite{xiang2017subcategory} & 0.10 & 0.07 & 0.27 & 0.18 & 0.41 & 0.28 \\
    & VIS-FNL~\cite{dongFS17} & 0.16 & 0.11 & 0.40 & 0.29 & 0.58 & 0.42 \\
     & OrcVIO & \bf{0.39} & \bf{0.43} & \bf{0.61} & \bf{0.69} & \bf{0.70} & \bf{0.78}  \\
    \bottomrule
  \end{tabular}
\end{table*}

\begin{table*}[!tbp]
  \centering
 \caption {Trajectory RMSE (m) on the KITTI Odometry Sequences} 
	\begin{tabular}{l rrrrrrrrrr}
    \toprule
KITTI Sequence $\rightarrow$  & 00 & 02 & 04 & 05 & 06 & 07 & 08 & 09 & 10 & Mean \\ \midrule
Object BA~\cite{frost2018recovering} & 73.4 & 55.5 & 10.7 & 50.8 & 73.1 &  47.1 & 72.2 & 31.2 & 53.5 & 51.9 \\
CubeSLAM~\cite{cubeslam} & 13.9 & 26.2 & 1.1 &  \bf{4.8} &  7.0 &  2.7 & \bf{10.7} & 10.7 & 8.4 & 9.5 \\ 
OrcVIO & \bf{10.9} & \bf{18.9} & \bf{0.8} & 5.5 & \bf{4.5} & \bf{2.5} & 14.1 & \bf{6.6} & \bf{5.3} & \bf{7.7} \\
\bottomrule
  \end{tabular}
\label{tab:ate_odom}
\end{table*}

\begin{figure}[t]
\includegraphics[width=\linewidth]{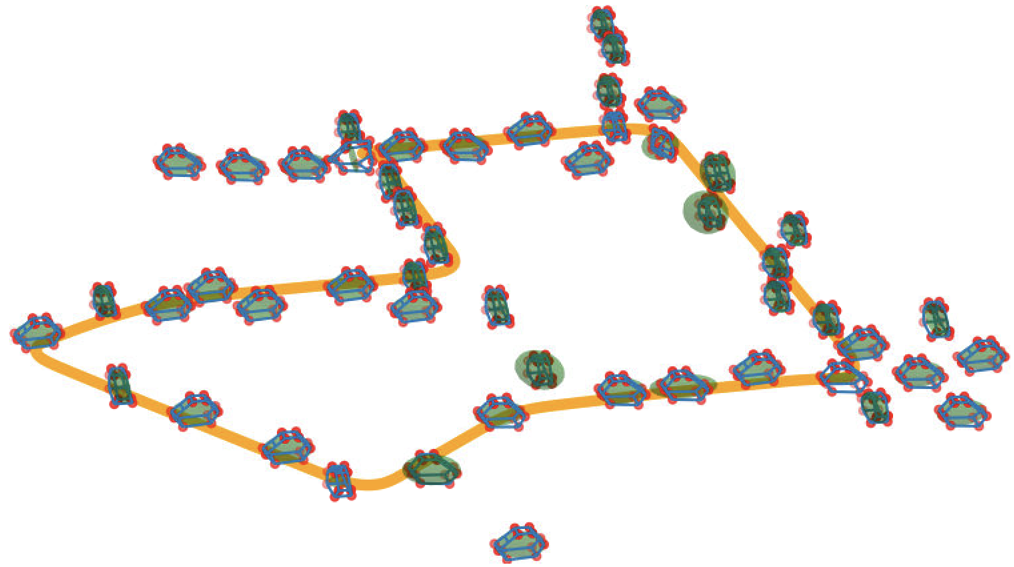}
\caption{Estimated IMU-camera trajectory (yellow) and object states (red landmarks and green ellipsoids) from KITTI Odom. Seq. 07.}
\label{fig:odom07}
\end{figure}



\begin{figure}[t]
\includegraphics[width=\linewidth]{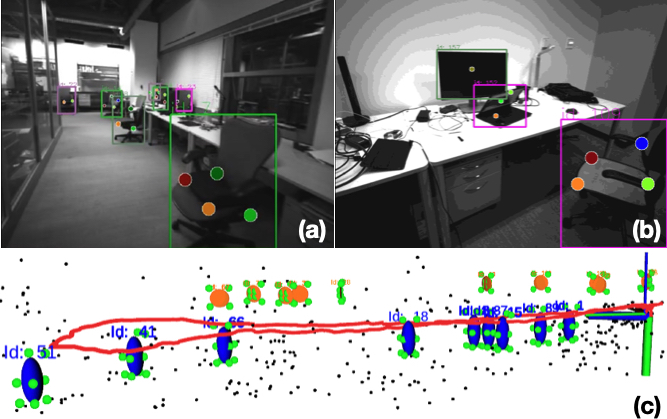}
\caption{OrcVIO localization and object mapping in an indoor scene with chairs and monitors. Bounding-box and semantic-keypoint detections are shown in (a) and (b). The estimated sensor trajectory (red curve), geometric landmarks (black dots), semantic landmarks (green dots), and object ellipsoids (blue for chairs, orange for monitors) obtained by OrcVIO are shown in (c).}
\label{fig:erl_lab_scene_map}
\end{figure}


\begin{figure*}[t]
\includegraphics[width=\linewidth]{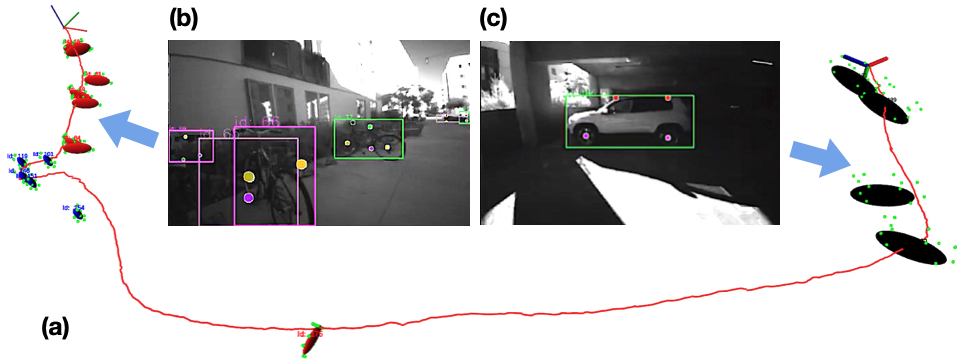}
\caption{OrcVIO localization and object mapping in an outdoor scene with chairs, bikes and cars. The estimated sensor trajectory (red curve), semantic landmarks (green dots), and object ellipsoids (blue for chairs, red for bikes, and black for cars) are shown in (a). Bounding-box and semantic-keypoint detections are shown for bicycle and car instances in (b) and (c).}
\label{fig:mesa_scene_map}
\end{figure*}

\subsection{KITTI Dataset Results}

We also evaluate OrcVIO on the KITTI dataset \cite{geiger2012we}, both qualitatively and quantitatively. We use the raw data sequences with object annotations to evaluate the object state estimation and the odometry sequences without object annotations for trajectory accuracy evaluation. Since the original inertial data from KITTI is not reliable, we use a VIO simulator \cite{geneva2019openvins} to generate realistic noisy high-frequency IMU measurements at 250 Hz from the ground-truth poses. 



Fig.~\ref{fig:odom07} shows the IMU-camera trajectory and object states estimated on KITTI odometry sequence 07, and a video demo is provided in the \hyperref[sec:supp]{\textcolor{Blue}{Supplementary Material}}. The results demonstrate that OrcVIO produces meaningful object maps. We obtained ground-truth 3D annotations from the KITTI tracklets and the KITTI detection benchmark labels for quantiative evaluation of the object reconstruction. Table~\ref{tab:3d_iou} reports 3D IoU results comparing OrcVIO
against state-of-the-art methods, including a deep learning approach for single-view 3-D bounding-box estimation (SingleView \cite{mousavian20173d}), and a multi-view bundle-adjustment algorithm that uses cuboids to represent objects (CubeSLAM \cite{cubeslam}). OrcVIO has better 3D IoU than SingleView for the majority of the sequences (23, 36, 39, 64, 96) because, unlike SingleView, OrcVIO use multi-view optimization over the object states. The performances of OrcVIO and CubeSLAM are similar since both rely on point and bounding-box measurements to optimize the object states. CubeSLAM is better than OrcVIO in terms of 3D IoU, possibly because OrcVIO does not model dynamic objects. Table~\ref{tab:3d_iou} also shows that OrcVIO is slightly better than CubeSLAM according to the TE odometry metric. In contrast with CubeSLAM, OrcVIO uses inertial residuals in addition to the geometric and object residuals.



In Table~\ref{tab:pr_kitti}, we compare the Precision and Recall of OrcVIO on the KITTI raw sequences (2011\_09\_26\_00XX, XX = [01, 19, 22, 23, 35, 36, 39, 61, 64, 93]) against a single-view, end-to-end object estimation approach (SubCNN \cite{xiang2017subcategory}), and a visual-inertial object detector (VIS-FNL \cite{dongFS17}). The first six rows are the Precision and Recall associated with different translation error (row) and rotation error (column) thresholds, whereas the last 3 rows ignore the rotation error. The results demonstrate that OrcVIO retrieves a reasonable amount of the ground-truth objects and provides a high-quality object map. When both rotation and translation errors are considered (first six rows), OrcVIO is better than SubCNN, since the latter does not rely on temporal object tracking. OrcVIO is comparable with VIS-FNL, even though VIS-FNL uses multiple object hypotheses while OrcVIO only keeps one object state. OrcVIO outperforms SubCNN and VIS-FNL when only translation error is taken into account (last three rows), which suggests that the object position estimates are accurate but the orientation estimates could be improved. 

We evaluate the RMSE of the IMU-camera trajectory estimation in Table~\ref{tab:ate_odom}. OrcVIO is compared with two visual object SLAM methods: a monocular visual SLAM that integrates spherical object models to estimate the scale via bundle-adjustment (Object BA \cite{frost2018recovering}), and CubeSLAM \cite{cubeslam}. Table~\ref{tab:ate_odom} shows that OrcVIO outperforms Object BA because spheres are a very coarse shape representation, compared to our ellipsoid and semantic keypoint representation, and thus Object BA cannot maintain the object scales as accurately as OrcVIO. The use of inertial data in OrcVIO leads to superior results as well.
OrcVIO also outperforms CubeSLAM, possibly due to the incorporation of a zero-velocity update (Sec.~\ref{sec:zero_velocity_update}), which is critical in driving datasets with frequent stops.


\subsection{UCSD Campus Dataset Results}

We also evaluated OrcVIO using real data collected with two commercial VIO sensors on UCSD's campus. The results are qualitative due to the lack of ground truth information.

First, we used Intel RealSense D435i with image frequency of 30 Hz, image resolution of $640\times480$, IMU frequency of 200 Hz in an indoor lab scene with chairs and monitors as shown in Fig. \ref{fig:erl_lab_scene_map}. The estimated sensor trajectory and reconstructed object map by OrcVIO are shown in the figure. A video demo can be found in the \hyperref[sec:supp]{\textcolor{Blue}{Supplementary Material}}. The results demonstrate that OrcVIO can map object instances from different categories and operates at real-time speed in a cluttered indoor scene. Since OrcVIO does not currently have a loop-closing mechanism for object re-identification, objects reappearing after getting lost will be mapped twice. Thus, there are more reconstructed chairs in Fig. \ref{fig:erl_lab_scene_map} than in the reality. 

Semantic keypoint detection is challenging due to occlusion, viewpoint change, and the lack of distinctive features on the monitors as shown in Fig. \ref{fig:erl_lab_scene_map} (b). To handle the monitor class succesfully, we decreased the weight of the semantic keypoint residual in the object Levenberg-Marquardt optimization \eqref{eq:object-lm}. Although removing the reliance on semantic keypoints leads to worse object orientation estimation, it allows OrcVIO to work with bounding-box detections only. This simplifies the front-end to an object detector and tracker and makes the algorithm more efficient and easier to deploy on resource-constrained robots. We release code for both the full OrcVIO algorithm and the OrcVIO-Lite version in the \hyperref[sec:supp]{\textcolor{Blue}{Supplementary Material}}.

Finally, we used an INDEMIND Binocular Visual-Inertial Camera to run OrcVIO outdoors with images at 25 Hz with resolution of $640\times400$ and IMU measurements at 200 Hz. The sensor initially observes bikes and chairs, and then makes a transition into a parking structure, as shown in Fig. \ref{fig:mesa_scene_map} (b), (c). A video demonstrating the performance of OrcVIO on this dataset can be found in the \hyperref[sec:supp]{\textcolor{Blue}{Supplementary Material}}. The resulting object map is shown in Fig. \ref{fig:mesa_scene_map} (a), demonstrating that OrcVIO is able to estimate object states from different categories in both indoor and outdoor scenes. This experiment also shows that OrcVIO can handle large illumination changes, transitioning from direct outdoor sunlight to dim lighting inside the parking structure.

%% file: tex/Conclusion.tex
\section{Conclusion}
\label{sec:conclusion}

This paper presented a unified formulation of ego-motion and object pose and shape estimation and developed a real-time simultaneous localization and object mapping algorithm. OrcVIO provides computationally constrained robot platforms with the ability to understand their surroundings at geometric and semantic levels, which may enable further capabilities such as object-level loop closure in challenging environments, collaborative semantic SLAM, and object interaction. Estimating object motion and performing collaborative mapping and loop closure are promising directions for future research.


%% file: tex/Appendix.tex
\section{Appendix}
\label{sec:appendix}

\subsection{Proof of Proposition \ref{prop:gk-jacobians}}
\label{sec:prop1_proof}

We first derive the Jacobian of the camera pose perturbation ${}^{}_{C}\bfxi = [{}^{}_{C}{\bfrho}^\top\ {}^{}_{C}\bftheta^\top]^\top$ with respect to the IMU pose perturbation ${}_{I}\bfxi = [{}^{}_{I}{\bftheta}^\top\ {}^{}_{I}\tilde{\bfp}^\top]^\top$ in \eqref{eq:dcxi-dixi}. The camera and IMU poses are related via ${}^{}_{C}\bfT = {}^{}_{I}\bfT {}^{I}_{C}\bfT$. Approximating ${}^{}_{C}\bfT$ using ${}^{}_{C}\hat{\bfT}$ and ${}^{}_{C}\bfxi$ leads to:
\begin{equation*}
  \label{eq:state_aug-jacobians-proof1b}
  \begin{aligned}
  {}^{}_{C}\bfT &\approx {}^{}_{C}\hat{\bfT} \prl{\bfI_4 + {}^{}_{C}\bfxi_\times} 
  = 
  \begin{bmatrix}
  {}^{}_{C}\hat{\bfR} & {}^{}_{C}\hat{\bfp} \\
  \bf0 & 1
  \end{bmatrix} +
  \begin{bmatrix}
  {}^{}_{C}\hat{\bfR} {}^{}_{C}{\bftheta}_\times & {}^{}_{C}\hat{\bfR} {}^{}_{C}\boldsymbol{\rho} \\ 
  \bf0 & \bf0
  \end{bmatrix}.
  \end{aligned}
\end{equation*}
Similarly, approximating ${}^{}_{I}\bfT$ using ${}^{}_{I}\hat{\bfT}$ and ${}^{}_{I}\bfxi$ leads to:
\begin{equation*}
\label{eq:state_aug-jacobians-proof1c}
\begin{aligned}
{}^{}_{I}\bfT {}^{I}_{C}\bfT 
&\approx
\begin{bmatrix}
{}^{}_{I}\hat{\bfR} (\bfI_3 + {}^{}_{I}{\bftheta}_\times) &  {}^{}_{I}\hat{\bfp} + {}^{}_{I}\tilde{\bfp} \\ \bf0 & 1
\end{bmatrix}
\begin{bmatrix}
{}^{I}_{C}\bfR &  {}^{I}_{C}\bfp \\ \bf0 & 1
\end{bmatrix} \\ 
&= 
\begin{bmatrix}
{}^{}_{C}\hat{\bfR} & {}^{}_{C}\hat{\bfp} \\
\bf0 & 1
\end{bmatrix}
+
\begin{bmatrix}
{}^{}_{I}\hat{\bfR} {}^{}_{I}{\bftheta}_\times {}^{I}_{C}\bfR & {}^{}_{I}\hat{\bfR} {}^{}_{I}{\bftheta}_\times {}^{I}_{C}\bfp + {}^{}_{I}\tilde{\bfp}\\ 
\bf0 & \bf0 
\end{bmatrix}.  
\end{aligned}
\end{equation*}
Equating the two expression above, we get:
\begin{equation}
\begin{aligned}
{}^{}_{C}\hat{\bfR} {}^{}_{C}{\bftheta}_\times 
&= 
{}^{}_{I}\hat{\bfR} {}^{I}_{C}\bfR {}^{}_{C}{\bftheta}_\times = {}^{}_{I}\hat{\bfR} {}^{}_{I}{\bftheta}_\times {}^{I}_{C}\bfR\\
{}^{}_{I}\hat{\bfR} {}^{I}_{C}\bfR {}^{}_{C}\boldsymbol{\rho} 
&=
{}^{}_{I}\hat{\bfR} {}^{}_{I}{\bftheta}_\times {}^{I}_{C}\bfp + {}^{}_{I}\tilde{\bfp}
\end{aligned}
\end{equation}
and \eqref{eq:dcxi-dixi} can be concluded from:
\begin{equation}
\begin{aligned}
{}^{}_{C}{\bftheta}_\times 
&=
{}^{I}_{C}\bfR^\top {}^{}_{I}{\bftheta}{}^{I}_{C}\bfR =
\brl{{}^{I}_{C}\bfR^\top {}^{}_{I}{\bftheta}}_\times 
\;\Rightarrow\;
{}^{}_{C}{\bftheta} =
{}^{I}_{C}\bfR^\top {}^{}_{I}{\bftheta},\\
{}^{}_{C}\boldsymbol{\rho} 
&=
- {}^{I}_{C}\bfR^\top {}^{I}_{C}\bfp_\times {}^{}_{I}{\bftheta} + {}^{}_{C}\hat{\bfR}^\top {}^{}_{I}\tilde{\bfp}.
\end{aligned}
\end{equation}

Next, we derive the expressions in \eqref{eq:gk-jacobians}. Note that:
\begin{equation} \label{eq:dcTldxi}
\begin{aligned}
{}_C\bfT^{-1} \underline{\bfell} &\approx 
 (\bfI_4 + {}_C\bfxi_{\times})^{-1} {}_C\hat{\bfT}^{-1}\underline{\bfell}
\approx (\bfI_4 - {}_C\bfxi_{\times}) {}_C\hat{\bfT}^{-1} \underline{\bfell}  \\
&= {}_C\hat{\bfT}^{-1} \underline{\bfell} - \brl{{}_C\hat{\bfT}^{-1} \underline{\bfell}}^{\odot} {}_C\bfxi,
\end{aligned}
\end{equation}
where $(\bfI_4 + \bfxi_\times)(\bfI_4 - \bfxi_\times) = \bfI_4 - \bfxi_\times + \bfxi_\times - (\bfxi_\times)^2 \approx \bfI_4$ by discarding high-order terms in $\bfxi$, and we used (7.159) in \cite[Ch.7]{BarfootBook} in the last equality. Applying the chain rule to \eqref{eq:gk-error}:
\begin{equation} \label{eq:dgedixi}
\frac{\partial{}^g\bfe}{\partial{}_{I}\bfxi} = \bfP  \frac{d\pi}{d\underline{\bfs}}\prl{ {}_C\bfT^{-1} \underline{\bfell}} \frac{\partial {}_C\bfT^{-1} \underline{\bfell}}{\partial{}_{C}\bfxi}  \frac{\partial {}_{C}\bfxi}{\partial{}_{I}\bfxi},
\end{equation}
and evaluating at ${}_C\hat{\bfT}$ and $\hat{\bfell}$, using $\frac{\partial {}_C\bfT^{-1} \underline{\hat{\bfell}}}{\partial{}_{C}\bfxi} = - \brl{{}_C\hat{\bfT}^{-1} \underline{\hat{\bfell}}}^{\odot}$ from \eqref{eq:dcTldxi}, leads to the first equation in \eqref{eq:gk-jacobians}. The second equation in \eqref{eq:gk-jacobians} follows directly from the chain rule:
\begin{equation} \label{eq:duldl}
\frac{\partial{}^g\bfe}{\partial \bfell} = \bfP  \frac{d\pi}{d\underline{\bfs}}\prl{ {}_C\bfT^{-1} \underline{\bfell}} {}_C\bfT^{-1} \frac{d\underline{\bfell}}{d \bfell}, \qquad
\frac{d\underline{\bfell}}{d\bfell} = \begin{bmatrix} \bfI_{3} \\ \mathbf{0}^\top \end{bmatrix}.
\end{equation}

\subsection{Proof of Proposition \ref{prop:sk-jacobians}}
\label{sec:prop2_proof}

The derivation of $\frac{\partial{}^s\bfe}{\partial{}_{I}\bfxi}$ is identical to the derivation of $\frac{\partial{}^g\bfe}{\partial{}_{I}\bfxi}$ in \eqref{eq:dgedixi} with $\underline{\bfell}$ replaced by ${}_O\bfT\prl{\underline{\bfs}_l + \delta\underline{\bfs}_l}$, which does not depend on ${}_{I}\bfxi$. To derive the Jacobians of ${}^s\bfe$ with respect to ${}_{O}\bfxi$ and $\delta\tilde{\bfs}_l$, apply the chain rule to \eqref{eq:sk-error}:
\begin{align}
\frac{\partial{}^s\bfe}{\partial{}_{O}\bfxi} &=   \bfP \frac{d\pi}{d\underline{\bfs}}\prl{{}_C\bfT^{-1} {}_O\bfT\prl{\underline{\bfs}_l + \delta\underline{\bfs}_l}}  {}_C\bfT^{-1} \frac{\partial {}_O\bfT\prl{\underline{\bfs}_l + \delta\underline{\bfs}_l}}{\partial{}_{O}\bfxi} \notag\\
\frac{\partial{}^s\bfe}{\partial\delta\tilde{\bfs}_l} &= \bfP \frac{d\pi}{d\underline{\bfs}}\prl{{}_C\bfT^{-1} {}_O\bfT\prl{\underline{\bfs}_l + \delta\underline{\bfs}_l}}  {}_C\bfT^{-1} {}_O\bfT \frac{d \delta\tilde{\underline{\bfs}}_l}{d\delta\tilde{\bfs}_l}.
\end{align}
As in \eqref{eq:duldl}, $\frac{d \delta\tilde{\underline{\bfs}}_l}{d\delta\tilde{\bfs}_l} = \begin{bmatrix} \bfI_{3} \\ \mathbf{0}^\top \end{bmatrix}$, while $\frac{\partial {}_O\bfT\prl{\underline{\bfs}_l + \delta\underline{\bfs}_l}}{\partial{}_{O}\bfxi}$, evaluated at ${}^{}_{O}\hat{\bfT}$ and $\delta\hat{\bfs}_l$, can be obtained as:
\begin{align}
{}_O\bfT\prl{\underline{\bfs}_l + \delta\underline{\hat{\bfs}}_l}
&\approx 
{}_O\hat{\bfT} (\bfI_4 + {}_O\bfxi_\times) \prl{\underline{\bfs}_l + \delta\underline{\hat{\bfs}}_l} \\
&=
{}_O\hat{\bfT} \prl{\underline{\bfs}_l + \delta\underline{\hat{\bfs}}_l} + \underbrace{{}_O\hat{\bfT} \brl{ \prl{\underline{\bfs}_l + \delta\underline{\hat{\bfs}}_l}}^\odot}_{\text { Jacobian }} {}_O\bfxi. \notag
\end{align}
\vspace*{-3ex}

\subsection{Proof of Proposition \ref{prop:bb-jacobians}}
\label{sec:prop3_proof}

\begin{lemma}
\label{prop:ellipsoid_sdf}
Given a hyperplane $\ubfb = [\bfb^\top \, - b_h]^\top$ and an origin-centered ellispoid as a dual quadric $\bfQ^* = \diag(\bfU^{2}, -1)$, there are two hyperplanes parallel to $\ubfb$ that are tangent to the ellipsoid:
$
\ubft = 
\begin{bmatrix}
\bfb^\top & \pm \sqrt{ \bfb^\top \bfU^{2} \bfb }
\end{bmatrix}^\top
$
and the signed distance between $\ubfb$ and the closest tangent hyperplane is:
\begin{equation} \label{eq:tangent-distance}
\begin{aligned}
d(\ubfb,\bfU) \triangleq 
\frac{1}{\|\bfb\|}\prl{\sgn(b_h) \sqrt{ \bfb^\top \bfU^{2} \bfb } - b_h}.
\end{aligned}
\end{equation}
\end{lemma}
\begin{proof}
 
Let the tangent parallel to $\ubfb$ be 
$ \ubft = \begin{bmatrix} \bfb^\top & - \alpha \end{bmatrix}^\top$.
Recall that the normal to an isosurface at any point $\bfx$ is the gradient of the isosurface $\nabla_\bfx \bfx^\top \bfU^{-2} \bfx = 2\bfU^{-2}\bfx$. For a tangent plane $\bft$, the plane normal and ellipsoid normal must be in the same direction, $\bfb \propto 2\bfU^{-2}\bfx$. Assuming $\beta \in \bbR$ as unknown constant, we can solve for $\bfx = \beta \bfU^{2} \bfb$. Because $\bfx$ lies on the ellipsoid $\beta^2 \bfb^\top \bfU^{2}\bfb  = 1$ or $\beta = \pm \frac{1}{\sqrt{\bfb^\top \bfU^{2} \bfb}}$.
If $\bfx$ also lies on the hyperplane $\ubft$, then $0 = \bfb^\top \bfx - \alpha = \beta \bfb^\top \bfU^{2} \bfb  - \alpha = \pm \sqrt{\bfb^\top \bfU^{2} \bfb} - \alpha$, and we get $\alpha = \pm \sqrt{ \bfb^\top \bfU^{2} \bfb }$. The perpendicular distance between two parallel hyperplanes $\brl{\bfb^\top\, -\alpha_1}^\top$ and $\brl{\bfb^\top\, -\alpha_2}^\top$ is $\frac{\alpha_1 - \alpha_2}{\|\bfb\|}$. The distance from the nearest tangent $\ubft$ to the hyperplane $\ubfb$ is 
$
d(\ubfb,\bfU) = 
\frac{ 1 }{\|\bfb\|}\min_\alpha (\alpha - b_h ) 
=  \frac{1}{\|\bfb\|}\prl{\sgn(b_h) \sqrt{ \bfb^\top \bfU^{2} \bfb } - b_h}. 
$
The signed distance is chosen such that subtracting the residual from $\ubft$ takes it closer to $\ubfb$. 
\end{proof}

\begin{lemma} \label{prop:ellipsoid_sdf_jacobians}
Given a hyperplane $\ubfb = [\bfb^\top \, - b_h]^\top$ and an origin-centered ellispoid as a dual quadric $\bfQ^* = \diag(\bfU^{2}, -1)$, the Jacobian of the distance $d(\ubfb,\bfU)$ in \eqref{eq:tangent-distance} with respect to $\ubfb$ is:
\begin{equation*}
\scaleMathLine{\frac{\partial d(\ubfb,\bfU)}{\partial \ubfb}
= 
\frac{\sgn(b_h)\bfb^\top\bfU^2 }{\|\bfb\| \sqrt{\bfb^\top \bfU^2 \bfb}}
\begin{bmatrix}
\left(
 \bfI_{3}
- \frac{\bfb\bfb^\top}{\|\bfb\|^2}
\right)
&
0
\end{bmatrix}
+ \frac{\brl{b_h\bfb^\top\  \|\bfb\|^2} }{\|\bfb\|^3}.}
\end{equation*}
%
\end{lemma}

\begin{proof}
We rewrite $d(\ubfb,\bfU)$ explicitly in terms of $\ubfb$ by replacing $b_h = -\brl{\mathbf{0}^\top\ 1}\ubfb = - \underline{\mathbf{0}}^\top \ubfb$:
\begin{equation*}
d(\ubfb,\bfU)
= \frac{1}{\|\bfb\|}
\underbrace{
\left(
- \sgn(\underline{\bf0}^\top \ubfb) \sqrt{\ubfb^\top \diag(\bfU^2, 0) \ubfb}
+ \underline{\bf0}^\top \ubfb
\right)
}_{\triangleq d_1(\ubfb, \bfU)}.
\end{equation*}
Taking the derivative using the product rule leads to:
\begin{equation} \label{eq:bbox-jac-prod-rule}
\frac{\partial 
d(\ubfb,\bfU) }{\partial \ubfb}
=  
\frac{1}{\|\bfb\|}
\frac{\partial }{\partial \ubfb} d_1(\ubfb, \bfU)
+ 
d_1(\ubfb, \bfU)
\frac{\partial }{\partial \ubfb} \frac{1}{\|\bfb\|}.
\end{equation}
The derivative of $\frac{1}{\|\bfb\|}$ is:
\begin{equation*}
\begin{aligned}
\frac{\partial }{\partial \ubfb} \frac{1}{\|\bfb\|}
&=  
\frac{\partial }{\partial \ubfb} \frac{1}{\sqrt{\bfb^\top \bfb}}
=
\frac{\partial }{\partial \ubfb} \left( \ubfb^\top \diag(\bfI_{3}, 0) \ubfb\right)^{-\frac{1}{2}}
\\
&=  
\left(-\frac{1}{2(\bfb^\top \bfb)^{\frac{3}{2}}} \right)\frac{\partial \ubfb^\top \diag(\bfI_{3} , 0)\ubfb}{\partial \ubfb}
\\
&=  
-  \frac{\ubfb^\top \diag(\bfI_{3} , 0)}{\|\bfb\|^3} 
= 
-  \frac{\begin{bmatrix}\bfb^\top&  0\end{bmatrix}}{\|\bfb\|^3}.
\end{aligned}
\end{equation*}
The derivative of $d_1(\ubfb, \bfU)$ is:
\begin{equation*}
\begin{aligned}
\frac{\partial}{\partial \ubfb} d_1(\ubfb, \bfU)
&= - 
\sgn(\underline{\mathbf{0}}^\top \ubfb)
\frac{\ubfb^\top \diag(\bfU^2, 0)}
{\sqrt{\ubfb^\top \diag(\bfU^2, 0) \ubfb}}
+ \underline{\mathbf{0}}^\top.
\\
&\quad= 
\sgn(b_h)
\frac{\brl{\bfb^\top \bfU^2\ 0}}
{\sqrt{\bfb^\top \bfU^2 \bfb}}
+ \underline{\mathbf{0}}^\top 
.
\end{aligned}
\end{equation*}
Putting the derivatives in \eqref{eq:bbox-jac-prod-rule} leads to:
\begin{equation*}
\begin{aligned}
&\frac{\partial  d(\ubfb,\bfU) }{\partial \ubfb}
= 
\frac{
\|\bfb\|^2 \underline{\mathbf{0}}^\top
+ 
b_h\brl{\bfb^\top\  0}
}{\|\bfb\|^3}
\\
&\qquad+\sgn(b_h)
\frac{
\|\bfb\|^2 \brl{\bfb^\top \bfU^2\  0}
-\bfb^\top\bfU^2\bfb\brl{\bfb^\top\  0} 
}
{\|\bfb\|^3 \sqrt{\bfb^\top \bfU^2 \bfb} }
\\
&=
\frac{\sgn(b_h)\bfb^\top\bfU^2 }{\|\bfb\| \sqrt{\bfb^\top \bfU^2 \bfb}}
\begin{bmatrix}
\left(
 \bfI_{3}
- \frac{\bfb\bfb^\top}{\|\bfb\|^2}
\right)
&
0
\end{bmatrix}
+ \frac{\brl{b_h\bfb^\top\  \|\bfb\|^2} }{\|\bfb\|^3}.\qedhere
\end{aligned}
\end{equation*}
\end{proof}


We proceed with the proof of Proposition \ref{prop:bb-jacobians}. Note that ${}^b\bfe(\bfx, \bfo, {}^b\bfz) = d(\ubfb,\diag(\bfu+\delta \bfu))$ in \eqref{eq:tangent-distance} with $\underline{\bfb}= {}_O\bfT^\top  {}_C\bfT^{-\!\top} \bfP^\top {}^b\underline{\mathbf{z}}$. The expression for $\frac{\partial {}^b\bfe}{\partial \underline{\bfb}}$, evaluated at $\underline{\hat{\bfb}}$ and $\hat{\bfU}$, follows directly from Lemma \ref{prop:ellipsoid_sdf_jacobians}. The expressions for $\frac{\partial \underline{\bfb}}{\partial {}_C\bfxi}$ and $\frac{\partial \underline{\bfb}}{\partial {}_O\bfxi}$ are obtained using the perturbations ${}_C\bfT \approx {}_C\hat{\bfT}(\bfI_4+{}_C\bfxi_{\times})$ and ${}_O\bfT \approx {}_O\hat{\bfT}(\bfI_4+{}_O\bfxi_{\times})$, (7.159) in \cite[Ch.7]{BarfootBook}, and dropping second-order perturbation terms: 
\begin{equation*}
\begin{aligned}
\ubfb &\approx \prl{\bfI + {}_O\bfxi_\times^\top} {}_O\hat{\bfT}^\top {}_C\hat{\bfT}^{-\!\top} \prl{\bfI - {}_C\bfxi_{\times}^\top}\bfP^\top {}^b\underline{\mathbf{z}}\\
&\approx {}_O\hat{\bfT}^\top {}_C\hat{\bfT}^{-\!\top} \bfP^\top {}^b\underline{\mathbf{z}} + \underbrace{\brl{{}_O\hat{\bfT}^\top {}_C\hat{\bfT}^{-\top} \bfP^\top {}^b\underline{\mathbf{z}}}^{\circledcirc \top}}_{ \partial \underline{\bfb} / \partial {}_O\bfxi } {}_O\bfxi\\
&\quad\underbrace{-{}_O\hat{\bfT}^\top {}_C\hat{\bfT}^{-\top} \brl{\bfP^\top {}^b\underline{\mathbf{z}}}^{\circledcirc \top}}_{ \partial \underline{\bfb} / \partial {}_C\bfxi} {}_C\bfxi
\end{aligned}
\end{equation*}
Finally, $\frac{\partial {}^b\bfe}{\partial \delta\tilde{\bfu}}$ is obtained by the chain rule:
\begin{equation*}
\scaleMathLine{\frac{\partial {}^b\bfe}{\partial \delta\tilde{\bfu}} = \frac{\sgn(\hat{b}_h)}{2\|\hat{\bfb}\| \sqrt{\hat{\bfb}^\top\!\diag(\bfu\!+\!\delta\hat{\bfu})^2 \hat{\bfb}}} \frac{\partial}{\partial \delta\tilde{\bfu}} \hat{\bfb}^\top \!\diag(\bfu\!+\!\delta\hat{\bfu}\!+\!\delta\tilde{\bfu})^2 \hat{\bfb}.}
\end{equation*}

\subsection{Proof of Proposition \ref{prop:closed_form_mean_prop}}
\label{sec:imu-integral-proof}

\begin{lemma} \label{lemma:JH-rodrigues}
Let $\bfomega \in \bbR^3$, $\bfJ_L(\bfomega) \triangleq \sum_{n=0}^\infty \frac{\bfomega_\times^n}{(n+1)!}$, and $\bfH_L(\bfomega) \triangleq \sum_{n=0}^\infty \frac{\bfomega_\times^n}{(n+2)!}$. For $\bfomega \neq 0$, $\bfJ_L(\bfomega)$ and $\bfH_L(\bfomega)$ admit closed-form expressions, shown in \eqref{eq:JH-rodrigues}.
\end{lemma}

\begin{proof}
Using that $\bfomega_\times^{2n+1} = (-1)^n\|\bfomega\|^{2n} \bfomega_\times$ for $n \geq 0$:
\begin{align*}
&\bfJ_L(\bfomega) = \bfI_3 + \sum_{n=1}^\infty \frac{\bfomega_\times^n}{(n+1)!}\\
&= \bfI_3 + \prl{\sum_{n=0}^\infty \frac{(-1)^n\|\bfomega\|^{2n}}{(2n+2)!}} \bfomega_\times + \prl{\sum_{n=0}^\infty \frac{(-1)^n\|\bfomega\|^{2n}}{(2n+3)!}}  \bfomega_\times^2\\
&= \bfI_3 + \prl{\frac{1-\cos\|\bfomega\|}{\|\bfomega\|^2}}\bfomega_\times + \prl{\frac{\|\bfomega\|-\sin\|\bfomega\|}{\|\bfomega\|^3}}\bfomega_\times^2.
\end{align*}
The derivation for $\bfH_L(\bfomega)$ is equivalent.
\end{proof}

\begin{lemma} \label{lemma:SO3-integration}
For $\bfomega \in \bbR^3$, $t \in \bbR$, the matrix $\exp(t\bfomega_\times)$ satisfies:
\begin{equation}
\begin{aligned}
    \int_0^t \int_0^s \exp(r\bfomega_\times)dr ds = \!\int_0^t \!s \bfJ_L(s \bfomega) ds = t^2 \bfH_L(t \bfomega).
\end{aligned}
\end{equation}
\end{lemma}
\begin{proof}
The result follows by integrating the terms of the Taylor series of $\exp(t\bfomega_\times) = \sum_{n=0}^\infty \frac{t^n\bfomega_\times^n}{n!}$ and since $\bfJ_L(\bfomega)$ and $\bfH_L(\bfomega)$ are well defined by Lemma \ref{lemma:JH-rodrigues}.
\end{proof}

To obtain \eqref{eq:imu-integral}, we compute the solutions to the linear time-invariant (LTI) ordinary differential equations (ODEs) in \eqref{eq:imu-error-dynamics-nominal}. With $\bfomega = {}^i\bfomega_k - \hat{\bfb}_{g,k}$ and $t \in [0,\tau_k)$, the solution to:
\begin{equation}\label{eq:R-ode}
{}_I\dot{\hat{\bfR}} = {}_I\hat{\bfR}\bfomega_\times, \qquad {}_I\hat{\bfR}(0) = {}_I\hat{\bfR}_k,
\end{equation}
is ${}_I\hat{\bfR}(t) = {}_I\hat{\bfR}_k \exp(t \bfomega_\times)$ and, hence:
\begin{equation}
{}_I\hat{\bfR}_{k+1}^p = {}_I\hat{\bfR}(\tau_k) = {}_I\hat{\bfR}_k \exp\prl{\tau_k \bigl({}^i\bfomega_k \!- \hat{\bfb}_{g,k}\bigr)_{\times}}.
\end{equation}
Similarly, with $\bfa = {}^i\bfa - \hat{\bfb}_a$, $t \in [0,\tau_k)$, and initial condition ${}_I\hat{\bfv}(0) = {}_I\hat{\bfv}_{k}$, the solution of the LTI ODE for ${}_I\hat{\bfv}$ in \eqref{eq:imu-error-dynamics-nominal} is:
\begin{equation}
\begin{aligned}
{}_I\hat{\bfv}(t) &= {}_I\hat{\bfv}(0) + \int_0^t ({}_I\hat{\bfR}(s) \bfa + \bfg)\, ds \\
&= {}_I\hat{\bfv}_{k} + t {}_I\hat{\bfR}_k \bfJ_L(t\bfomega)\bfa + t\bfg,
\end{aligned}  
\end{equation}
where the second equality uses Lemma~\ref{lemma:SO3-integration}. Hence, ${}_I\hat{\bfv}_{k+1}^p ={}_I\hat{\bfv}(\tau_k)$ satisfies the second expression in \eqref{eq:imu-integral}. Also, by Lemma~\ref{lemma:SO3-integration}, for $t \in [0,\tau_k)$ with initial condition ${}_I\hat{\bfp}(0) = {}_I\hat{\bfp}_{k}$, the solution of the LTI ODE for ${}_I\hat{\bfp}$ in \eqref{eq:imu-error-dynamics-nominal} is:
\begin{equation}
\begin{aligned}
{}_I\hat{\bfp}(t) &= {}_I\hat{\bfp}(0) + \int_0^t {}_I\hat{\bfv}(s)\, ds \\
&= {}_I\hat{\bfp}_{k} + t {}_I\hat{\bfv}_{k} + t^2 {}_I\hat{\bfR}_k \bfH_L(t\bfomega)\bfa + \frac{t^2}{2}\bfg.
\end{aligned}  
\end{equation}
Hence, ${}_I\hat{\bfp}_{k+1}^p ={}_I\hat{\bfp}(\tau_k)$ satisfies the third expression in \eqref{eq:imu-integral}. Finally, $\hat{\bfb}_g(t) = \hat{\bfb}_g(0) = \hat{\bfb}_{g,k}$ and $\hat{\bfb}_a(t) = \hat{\bfb}_a(0) = \hat{\bfb}_{a,k}$ for all $t \in [0,\tau_k)$. The IMU pose history in \eqref{eq:imu-integral} is updated by adding the pose $({}_I\hat{\bfR}_k, {}_I\hat{\bfp}_{k})$ to the sliding window and dropping the oldest pose ${}_I\hat{\bfT}_{k-W}$.

\subsection{Proof of Proposition \ref{prop:closed_form_cov_prop}}
\label{sec:closed_form_cov_prop}

The transition matrix $\bfPhi(t,0)$ of \eqref{eq:smsckf_eq2} can be determined by computing the solution ${}^{}_{I}\tilde{\mathbf{x}}(t) = \bfPhi(t,0){}^{}_{I}\tilde{\mathbf{x}}(0)$ to the homogeneous system ${}^{}_{I}\dot{\tilde{\mathbf{x}}} = \bfF(t) {}^{}_{I}\tilde{\mathbf{x}}$ for an arbitrary initial condition ${}^{}_{I}\tilde{\mathbf{x}}(0) = ({}^{}_{I}\bftheta(0), {}^{}_{I}\tilde{\bfv}(0), {}^{}_{I}\tilde{\bfp}(0), \tilde{\bfb}_g(0), \tilde{\bfb}_a(0))$. Since the last two rows of $\bfF(t)$ are zero, the bias terms remain constant in the homogeneous system:
\begin{equation} \label{eq:phi-b}
\begin{aligned}
    \tilde{\bfb}_g(t) &= \tilde{\bfb}_g(0) = \begin{bmatrix}\mathbf{0} &\mathbf{0} &\mathbf{0} &\bfI_3 &\mathbf{0} \end{bmatrix} {}^{}_{I}\tilde{\mathbf{x}}(0),\\
    \tilde{\bfb}_a(t) &= \tilde{\bfb}_a(0) = \begin{bmatrix}\mathbf{0} &\mathbf{0} &\mathbf{0}  &\mathbf{0} &\bfI_3\end{bmatrix} {}^{}_{I}\tilde{\mathbf{x}}(0).
\end{aligned}
\end{equation}
Next, consider ${}^{}_{I}\dot{\bftheta}(t) = -\bfomega_\times {}^{}_{I}\bftheta(t) - \tilde{\bfb}_g(t)$ with $\bfomega = {}^i\bfomega_k - \hat{\bfb}_{g,k}$, which is a linear time-invariant (LTI) system in ${}^{}_{I}\bftheta(t)$. Using $\tilde{\bfb}_g(t) = \tilde{\bfb}_g(0)$ and Lemma~\ref{lemma:SO3-integration}, the system has solution:
\begin{align} \label{eq:phi-theta}
    {}^{}_{I}\bftheta(&t) = \exp(-t\bfomega_\times){}^{}_{I}\bftheta(0) - \int_0^t \exp(-(t-s)\bfomega_\times)ds \tilde{\bfb}_g(0) \notag\\
    &= \begin{bmatrix}\exp(-t\bfomega_\times) &\mathbf{0} &\mathbf{0} & -t \bfJ_L(-t \bfomega) &\mathbf{0} \end{bmatrix} {}^{}_{I}\tilde{\mathbf{x}}(0).
\end{align}
Next, consider ${}^{}_{I}\dot{\tilde{\bfv}}(t) = - {}^{}_{I}\hat{\bfR}(t) \bfa_{\times} {}^{}_{I}\bftheta(t) - {}^{}_{I}\hat{\bfR}(t)\tilde{\bfb}_a(t)$ with $\bfa = {}^i\bfa_k - \hat{\bfb}_{a,k}$, which is an LTI system in ${}^{}_{I}\tilde{\bfv}(t)$. Using that $\tilde{\bfb}_a(t) = \tilde{\bfb}_a(0)$, the LTI system has solution:
\begin{align} \label{eq:phi-v}
{}^{}_{I}\tilde{\bfv}(t) &= {}^{}_{I}\tilde{\bfv}(0) - \int_0^t {}^{}_{I}\hat{\bfR}(s) \bfa_{\times} {}^{}_{I}\bftheta(s) \, ds - \int_0^t {}^{}_{I}\hat{\bfR}(s) \, ds \tilde{\bfb}_a(0) \notag\\
&= \begin{bmatrix} \bfPhi_{\bfv\bftheta}(t) & \bfI_3 & \mathbf{0} &  \bfPhi_{\bfv\bfomega}(t) & \bfPhi_{\bfv\bfa}(t) \end{bmatrix}  {}^{}_{I}\tilde{\mathbf{x}}(0),
\end{align}
where:
\begin{align}
\bfPhi_{\bfv\bftheta}(t) &= - \int_0^t {}^{}_{I}\hat{\bfR}(s) \bfa_{\times} \exp(-s\bfomega_\times)\, ds, \label{eq:Xvth}\\
\bfPhi_{\bfv\bfomega}(t) &= \int_0^t s {}^{}_{I}\hat{\bfR}(s) \bfa_{\times} \bfJ_L(-s \bfomega)\, ds, \label{eq:Xvom}\\
\bfPhi_{\bfv\bfa}(t) &= - \int_0^t {}^{}_{I}\hat{\bfR}(s) \, ds = - t {}^{}_{I}\hat{\bfR}_k \bfJ_L(t \bfomega), \label{eq:Xva}
\end{align}
where \eqref{eq:Xva} follows from the solution to \eqref{eq:R-ode} and Lemma~\ref{lemma:SO3-integration}. To integrate \eqref{eq:Xvth}, we use the solution to \eqref{eq:R-ode}, the property $\bfR \bfa_{\times} \bfR^\top = [\bfR \bfa]_{\times}$ for $\bfR \in SO(3)$, $\bfa \in \bbR^3$, and Lemma~\ref{lemma:SO3-integration}:
\begin{align}
\bfPhi_{\bfv\bftheta}(t) &= - {}^{}_{I}\hat{\bfR}_k \int_0^t \exp(s\bfomega_\times) \bfa_{\times} \exp(s\bfomega_\times)^\top\, ds \\
&= - {}^{}_{I}\hat{\bfR}_k \brl{ \int_0^t \exp(s\bfomega_\times) \, ds \bfa }_{\times} = -t {}^{}_{I}\hat{\bfR}_k \brl{ \bfJ_L(t\bfomega) \bfa}_{\times}.\notag
\end{align}
To integrate \eqref{eq:Xvom}, observe that from $\bfomega_{\times}^3 = - \|\bfomega\|^2 \bfomega_{\times}$:
\begin{equation}
\bfI_3 - \frac{\bfomega_{\times}}{\|\bfomega\|^2} \prl{ \exp(\bfomega_{\times}) - \bfI_3 - \bfomega_{\times}} = \bfJ_L(\bfomega),
\end{equation}
and we can split \eqref{eq:Xvom} in two parts:
\begin{align} \label{eq:Xvom2}
\bfPhi_{\bfv\bfomega}(t) &= {}^{}_{I}\hat{\bfR}_k \int_0^t s \exp(s \bfomega_{\times}) \, ds\, \bfa_{\times} \prl{ \bfI_3 + \frac{\bfomega_{\times}^2}{\|\bfomega\|^2}} \\
&\quad + {}^{}_{I}\hat{\bfR}_k \int_0^t \exp(s \bfomega_{\times}) \frac{\bfa_{\times} \bfomega_{\times}}{\|\bfomega\|^2} \prl{\exp(s \bfomega_{\times})^\top - \bfI_3}  \, ds. \notag
\end{align}
By integrating the terms of the Taylor series of $s \exp(s \bfomega_{\times})$ and using $\bfomega_{\times}^3 = - \|\bfomega\|^2 \bfomega_{\times}$, we can verify that:
\begin{align} \label{eq:Delta}
\int_0^t &s \exp(s \bfomega_{\times}) \, ds =  \sum_{n=0}^\infty \frac{t^{n+2} \bfomega_{\times}^n}{(n+2)n!} =\frac{-1}{\|\bfomega\|^2} \sum_{n=0}^\infty \frac{t^{n+2} \bfomega_{\times}^{n+2}}{(n+2)n!} \notag\\
& = \frac{1}{\|\bfomega\|^2} \sum_{n=0}^\infty \prl{\frac{t^{n+2} \bfomega_{\times}^{n+2}}{(n+2)!} - \frac{t^{n+2} \bfomega_{\times}^{n+2}}{(n+1)!}} = \frac{\Delta(t)}{\|\bfomega\|^2}
\end{align}
where $\Delta(t)$ is defined in \eqref{eq:Phi-blocks}. The second integral in \eqref{eq:Xvom2} can be computed using $\bfa_{\times} \bfomega_{\times} = \bfomega \bfa^\top - (\bfa^\top \bfomega)\bfI_3$, $\exp(s \bfomega_{\times})\bfomega = \bfomega$, $\exp(s \bfomega_{\times})\exp(s \bfomega_{\times})^\top = \bfI_3$, and Lemma \ref{lemma:SO3-integration}, leading to:
\begin{align} \label{eq:Xvom3}
&\bfPhi_{\bfv\bfomega}(t) = {}^{}_{I}\hat{\bfR}_k \Delta(t) \frac{\bfa_{\times}}{\|\bfomega\|^2} \prl{ \bfI_3 + \frac{\bfomega_{\times}^2}{\|\bfomega\|^2}}\\
&\;\; + t {}^{}_{I}\hat{\bfR}_k\frac{\bfomega \bfa^\top}{\|\bfomega\|^2} \prl{\bfJ_L(-t\bfomega)-\bfI_3} + t{}^{}_{I}\hat{\bfR}_k\frac{\bfa^\top\bfomega}{\|\bfomega\|^2}\prl{\bfJ_L(t\bfomega)-\bfI_3}. \notag
\end{align}
Finally, consider ${}^{}_{I}\dot{\tilde{\bfp}}(t) = {}^{}_{I}\tilde{\bfv}(t)$, which is an LTI system in ${}^{}_{I}\tilde{\bfp}(t)$ with solution:
\begin{align} \label{eq:phi-p}
{}^{}_{I}\tilde{\bfp}(t) &= {}^{}_{I}\tilde{\bfp}(0) + \int_0^t {}^{}_{I}\tilde{\bfv}(s) \, ds \notag\\
&= \begin{bmatrix} \bfPhi_{\bfp\bftheta}(t) & t\bfI_3 & \bfI_3 &  \bfPhi_{\bfp\bfomega}(t) & \bfPhi_{\bfp\bfa}(t) \end{bmatrix}  {}^{}_{I}\tilde{\mathbf{x}}(0),
\end{align}
where:
\begin{align}
\bfPhi_{\bfp\bftheta}(t) &= \int_0^t \bfPhi_{\bfv\bftheta}(s) \, ds = - t^2{}^{}_{I}\hat{\bfR}_k \brl{ \bfH_L(t\bfomega) \bfa}_{\times}, \label{eq:Xpth}\\
\bfPhi_{\bfp\bfomega}(t) &= \int_0^t \bfPhi_{\bfv\bfomega}(s) \, ds, \label{eq:Xpom}\\
\bfPhi_{\bfp\bfa}(t) &= \int_0^t \bfPhi_{\bfv\bfa}(s) \, ds = - t^2{}^{}_{I}\hat{\bfR}_k \bfH_L(t\bfomega), \label{eq:Xpa}
\end{align}
where \eqref{eq:Xpth} and \eqref{eq:Xpa} follow from Lemma~\ref{lemma:SO3-integration}. To integrate \eqref{eq:Xpom}, we use that:
\begin{align}
\int_0^t \Delta(s)\, ds &= \int_0^t \exp(s\bfomega_{\times}) \, ds - \int_0^t s\exp(s\bfomega_{\times}) \, ds \bfomega_{\times} - t \bfI_3 \notag\\
&= t \bfJ_L(t\bfomega) - \frac{\bfomega_{\times} \Delta(t)}{\|\bfomega\|^2} - t \bfI_3,
\end{align}
where in the second integral we used \eqref{eq:Delta} and that $\bfomega_{\times}$ and $\Delta(t)$ commute. We integrate the second and third term in \eqref{eq:Xvom3} using Lemma \ref{lemma:SO3-integration} to obtain the final result for $\bfPhi_{\bfp\bfomega}(t)$ in \eqref{eq:Phi-blocks}. Since ${}^{}_{I}\tilde{\mathbf{x}}(0)$ was arbitrary, the rows of $\bfPhi(t,0)$ are provided by \eqref{eq:phi-theta}, \eqref{eq:phi-v}, \eqref{eq:phi-p}, and \eqref{eq:phi-b}.

\subsection{Zero-Velocity Update}
\label{sec:zero_velocity_update}

Zero-velocity conditions are frequently encountered in autonomous driving and autonomous flight applications, and thus the ability to determine whether the robot is static is important for reducing drift in VIO estimation \cite{geneva2019openvins, xiaochen2020lightweight}. 
In OrcVIO zero-velocity is detected similarly as in \cite{geneva2019openvins} by using pseudo zero inertial measurements and then checking the velocity magnitude.
For the zero-velocity update, we need to compute the residuals and the Jacobians of the inertial measurements with respect to the state.
Based on \eqref{eq:imu-dynamics} the residuals are:
\begin{equation}
\begin{aligned}
{}^{z}_{}\bfe(\bfx,{}^i\bfz)
&\triangleq
\left[\begin{array}{c}
\left(
{}^i\bfomega - \bfb_g 
\right)
-\boldsymbol{0}
\\
\left(
{}_I\bfR \prl{{}^i\bfa - \bfb_a} + \bfg
\right) 
-\boldsymbol{0}
\end{array}\right] 
\end{aligned}. 
\end{equation}
The corresponding Jacobians are presented as follows:
\begin{equation}
\begin{aligned}
\frac{\partial {}^{z}_{}\boldsymbol{e}}{\partial {}^{}_{I}{\bftheta}} 
&=
[\boldsymbol{0} \quad -{}_I\hat{\bfR} \left( {}^i\bfa - \hat{\bfb}_a \right)_\times]^\top,  \\ 
\frac{\partial {}^{z}_{}\boldsymbol{e}}{\partial \tilde{\bfb}_g} 
&= 
[-{\bfI}_3 \quad \boldsymbol{0}]^\top 
\quad 
\frac{\partial {}^{z}_{}\boldsymbol{e}}{\partial \tilde{\bfb}_a} 
= 
[\boldsymbol{0} \quad -{}_I\hat{\bfR}]^\top. 
\end{aligned}
\end{equation}

%% file: main.bbl
\begin{thebibliography}{10}
\providecommand{\url}[1]{#1}
\csname url@samestyle\endcsname
\providecommand{\newblock}{\relax}
\providecommand{\bibinfo}[2]{#2}
\providecommand{\BIBentrySTDinterwordspacing}{\spaceskip=0pt\relax}
\providecommand{\BIBentryALTinterwordstretchfactor}{4}
\providecommand{\BIBentryALTinterwordspacing}{\spaceskip=\fontdimen2\font plus
\BIBentryALTinterwordstretchfactor\fontdimen3\font minus
  \fontdimen4\font\relax}
\providecommand{\BIBforeignlanguage}[2]{{%
\expandafter\ifx\csname l@#1\endcsname\relax
\typeout{** WARNING: IEEEtran.bst: No hyphenation pattern has been}%
\typeout{** loaded for the language `#1'. Using the pattern for}%
\typeout{** the default language instead.}%
\else
\language=\csname l@#1\endcsname
\fi
#2}}
\providecommand{\BIBdecl}{\relax}
\BIBdecl

\bibitem{agarwal2009building}
S.~Agarwal, N.~Snavely, I.~Simon, S.~Seitz, and R.~Szeliski, ``{Building Rome
  in a Day},'' in \emph{IEEE International Conference on Computer Vision
  (ICCV)}, 2009.

\bibitem{schonberger2016structure}
J.~L. Schonberger and J.-M. Frahm, ``Structure-from-motion revisited,'' in
  \emph{IEEE Conference on Computer Vision and Pattern Recognition}, 2016, pp.
  4104--4113.

\bibitem{cadena2016past}
C.~Cadena, L.~Carlone, H.~Carrillo, Y.~Latif, D.~Scaramuzza, J.~Neira, I.~Reid,
  and J.~J. Leonard, ``Past, present, and future of simultaneous localization
  and mapping: Toward the robust-perception age,'' \emph{IEEE Transactions on
  Robotics}, vol.~32, no.~6, pp. 1309--1332, 2016.

\bibitem{orbslam2}
R.~{Mur-Artal} and J.~D. {Tardós}, ``{ORB-SLAM2: An Open-Source SLAM System
  for Monocular, Stereo, and RGB-D Cameras},'' \emph{IEEE Transactions on
  Robotics}, vol.~33, no.~5, pp. 1255--1262, Oct 2017.

\bibitem{msckf}
A.~I. {Mourikis} and S.~I. {Roumeliotis}, ``A multi-state constraint kalman
  filter for vision-aided inertial navigation,'' in \emph{IEEE International
  Conference on Robotics and Automation (ICRA)}, April 2007, pp. 3565--3572.

\bibitem{vinsmono}
T.~Qin, P.~Li, and S.~Shen, ``{VINS-Mono: A Robust and Versatile Monocular
  Visual-Inertial State Estimator},'' \emph{IEEE Transactions on Robotics},
  vol.~34, no.~4, pp. 1004--1020, Aug 2018.

\bibitem{bochkovskiy2020yolov4}
A.~Bochkovskiy, C.-Y. Wang, and H.-Y.~M. Liao, ``Yolov4: Optimal speed and
  accuracy of object detection,'' \emph{arXiv preprint: 2004.10934}, 2020.

\bibitem{he2017mask}
K.~{He}, G.~{Gkioxari}, P.~{Dollár}, and R.~{Girshick}, ``{Mask R-CNN},'' in
  \emph{IEEE International Conference on Computer Vision (ICCV)}, 2017, pp.
  2980--2988.

\bibitem{Wojke2018deep}
N.~Wojke and A.~Bewley, ``Deep cosine metric learning for person
  re-identification,'' in \emph{IEEE Winter Conference on Applications of
  Computer Vision (WACV)}, 2018, pp. 748--756.

\bibitem{lorbach2014prior}
M.~Lorbach, S.~H{\"o}fer, and O.~Brock, ``Prior-assisted propagation of spatial
  information for object search,'' in \emph{IEEE/RSJ International Conference
  on Intelligent Robots and Systems}, 2014, pp. 2904--2909.

\bibitem{sahin2020review}
C.~Sahin, G.~Garcia-Hernando, J.~Sock, and T.-K. Kim, ``A review on object pose
  recovery: from 3d bounding box detectors to full 6d pose estimators,''
  \emph{Image and Vision Computing}, p. 103898, 2020.

\bibitem{vasilopoulos2020reactive}
V.~Vasilopoulos, G.~Pavlakos, S.~Bowman, J.~D. Caporale, K.~Daniilidis, G.~J.
  Pappas, and D.~Koditschek, ``Reactive semantic planning in unexplored
  semantic environments using deep perceptual feedback,'' \emph{IEEE Robotics
  and Automation Letters}, 2020.

\bibitem{garg2020semantics}
S.~Garg, N.~S{\"u}nderhauf, F.~Dayoub, D.~Morrison, A.~Cosgun, G.~Carneiro,
  Q.~Wu, T.-J. Chin, I.~Reid, S.~Gould \emph{et~al.}, ``{Semantics for Robotic
  Mapping, Perception and Interaction: A Survey},'' \emph{Foundations and
  Trends in Robotics}, vol.~8, no. 1--2, pp. 1--224, 2020.

\bibitem{civera2008inverse}
J.~{Civera}, A.~J. {Davison}, and J.~M.~M. {Montiel}, ``Inverse depth
  parametrization for monocular slam,'' \emph{IEEE Transactions on Robotics},
  vol.~24, no.~5, pp. 932--945, Oct 2008.

\bibitem{civera20101}
J.~Civera, O.~G. Grasa, A.~J. Davison, and J.~M.~M. Montiel, ``{1-Point RANSAC
  for extended Kalman filtering: Application to real-time structure from motion
  and visual odometry},'' \emph{Journal of Field Robotics}, vol.~27, no.~5, pp.
  609--631, 2010.

\bibitem{svo2}
C.~{Forster}, Z.~{Zhang}, M.~{Gassner}, M.~{Werlberger}, and D.~{Scaramuzza},
  ``{SVO: Semidirect Visual Odometry for Monocular and Multicamera Systems},''
  \emph{IEEE Transactions on Robotics}, vol.~33, no.~2, pp. 249--265, April
  2017.

\bibitem{howard2008real}
A.~{Howard}, ``Real-time stereo visual odometry for autonomous ground
  vehicles,'' in \emph{IEEE/RSJ International Conference on Intelligent Robots
  and Systems (IROS)}, Sep. 2008, pp. 3946--3952.

\bibitem{smsckf}
K.~Sun, K.~Mohta, B.~Pfrommer, M.~Watterson, S.~Liu, Y.~Mulgaonkar, C.~J.
  Taylor, and V.~Kumar, ``Robust stereo visual inertial odometry for fast
  autonomous flight,'' \emph{IEEE Robotics and Automation Letters}, vol.~3,
  no.~2, pp. 965--972, April 2018.

\bibitem{newcombe2011dtam}
R.~A. {Newcombe}, S.~J. {Lovegrove}, and A.~J. {Davison}, ``Dtam: Dense
  tracking and mapping in real-time,'' in \emph{International Conference on
  Computer Vision (ICCV)}, Nov 2011, pp. 2320--2327.

\bibitem{engel2014lsd}
J.~Engel, T.~Sch{\"o}ps, and D.~Cremers, ``{LSD-SLAM: Large-Scale Direct
  Monocular SLAM},'' in \emph{European Conference on Computer Vision (ECCV)},
  2014, pp. 834--849.

\bibitem{gao2018ldso}
X.~{Gao}, R.~{Wang}, N.~{Demmel}, and D.~{Cremers}, ``Ldso: Direct sparse
  odometry with loop closure,'' in \emph{IEEE/RSJ International Conference on
  Intelligent Robots and Systems (IROS)}, Oct 2018, pp. 2198--2204.

\bibitem{weiss2011real}
S.~Weiss and R.~Siegwart, ``{Real-time metric state estimation for modular
  vision-inertial systems},'' in \emph{IEEE International Conference on
  Robotics and Automation (ICRA)}, 2011, pp. 4531--4537.

\bibitem{wu2015square}
K.~Wu, A.~Ahmed, G.~A. Georgiou, and S.~I. Roumeliotis, ``A square root inverse
  filter for efficient vision-aided inertial navigation on mobile devices,'' in
  \emph{Robotics: Science and Systems XI}, 2015.

\bibitem{okvis}
S.~Leutenegger, S.~Lynen, M.~Bosse, R.~Siegwart, and P.~Furgale,
  ``Keyframe-based visual–inertial odometry using nonlinear optimization,''
  \emph{The International Journal of Robotics Research}, vol.~34, no.~3, pp.
  314--334, 2015.

\bibitem{kendall2015posenet}
A.~{Kendall}, M.~{Grimes}, and R.~{Cipolla}, ``Posenet: A convolutional network
  for real-time 6-dof camera relocalization,'' in \emph{IEEE International
  Conference on Computer Vision (ICCV)}, Dec 2015, pp. 2938--2946.

\bibitem{zhou2017unsupervised}
T.~{Zhou}, M.~{Brown}, N.~{Snavely}, and D.~G. {Lowe}, ``Unsupervised learning
  of depth and ego-motion from video,'' in \emph{IEEE Conference on Computer
  Vision and Pattern Recognition (CVPR)}, July 2017, pp. 6612--6619.

\bibitem{clark2017vinet}
R.~Clark, S.~Wang, H.~Wen, A.~Markham, and N.~Trigoni, ``Vinet: Visual-inertial
  odometry as a sequence-to-sequence learning problem,'' in \emph{AAAI
  Conference on Artificial Intelligence}, 2017.

\bibitem{clark2017vidloc}
R.~{Clark}, S.~{Wang}, A.~{Markham}, N.~{Trigoni}, and H.~{Wen}, ``Vidloc: A
  deep spatio-temporal model for 6-dof video-clip relocalization,'' in
  \emph{IEEE Conference on Computer Vision and Pattern Recognition (CVPR)},
  July 2017, pp. 2652--2660.

\bibitem{wang2017deepvo}
S.~{Wang}, R.~{Clark}, H.~{Wen}, and N.~{Trigoni}, ``Deepvo: Towards end-to-end
  visual odometry with deep recurrent convolutional neural networks,'' in
  \emph{IEEE International Conference on Robotics and Automation (ICRA)}, May
  2017, pp. 2043--2050.

\bibitem{yin2018geonet}
Z.~{Yin} and J.~{Shi}, ``Geonet: Unsupervised learning of dense depth, optical
  flow and camera pose,'' in \emph{IEEE/CVF Conference on Computer Vision and
  Pattern Recognition (CVPR)}, June 2018, pp. 1983--1992.

\bibitem{li2018undeepvo}
R.~{Li}, S.~{Wang}, Z.~{Long}, and D.~{Gu}, ``Undeepvo: Monocular visual
  odometry through unsupervised deep learning,'' in \emph{IEEE International
  Conference on Robotics and Automation (ICRA)}, May 2018, pp. 7286--7291.

\bibitem{cao2020representations}
Y.~Cao, L.~Hu, and L.~Kneip, ``Representations and benchmarking of modern
  visual slam systems,'' \emph{Sensors}, vol.~20, no.~9, p. 2572, 2020.

\bibitem{galindo2005multi}
C.~Galindo, A.~Saffiotti, S.~Coradeschi, P.~Buschka, J.-A. Fernandez-Madrigal,
  and J.~Gonz{\'a}lez, ``Multi-hierarchical semantic maps for mobile
  robotics,'' in \emph{2005 IEEE/RSJ international conference on intelligent
  robots and systems}.\hskip 1em plus 0.5em minus 0.4em\relax IEEE, 2005, pp.
  2278--2283.

\bibitem{leibe2007dynamic}
B.~Leibe, N.~Cornelis, K.~Cornelis, and L.~Van~Gool, ``Dynamic 3d scene
  analysis from a moving vehicle,'' in \emph{2007 IEEE Conference on Computer
  Vision and Pattern Recognition}.\hskip 1em plus 0.5em minus 0.4em\relax IEEE,
  2007, pp. 1--8.

\bibitem{civera2011towards}
J.~Civera, D.~G{\'a}lvez-L{\'o}pez, L.~Riazuelo, J.~D. Tard{\'o}s, and J.~M.~M.
  Montiel, ``Towards semantic slam using a monocular camera,'' in \emph{2011
  IEEE/RSJ International Conference on Intelligent Robots and Systems}.\hskip
  1em plus 0.5em minus 0.4em\relax IEEE, 2011, pp. 1277--1284.

\bibitem{pronobis2011semantic}
A.~Pronobis, ``Semantic mapping with mobile robots,'' Ph.D. dissertation, KTH
  Royal Institute of Technology, 2011.

\bibitem{stuckler2015dense}
J.~St{\"u}ckler, B.~Waldvogel, H.~Schulz, and S.~Behnke, ``{Dense real-time
  mapping of object-class semantics from RGB-D video},'' \emph{Journal of
  Real-Time Image Processing}, vol.~10, no.~4, pp. 599--609, 2015.

\bibitem{vineet2015incremental}
V.~Vineet, O.~Miksik, M.~Lidegaard, M.~Nie{\ss}ner, S.~Golodetz, V.~A.
  Prisacariu, O.~K{\"a}hler, D.~W. Murray, S.~Izadi, P.~P{\'e}rez, and P.~H.~S.
  Torr, ``{Incremental dense semantic stereo fusion for large-scale semantic
  scene reconstruction},'' in \emph{IEEE International Conference on Robotics
  and Automation (ICRA)}, 2015, pp. 75--82.

\bibitem{pillai2015monocular}
S.~Pillai and J.~Leonard, ``Monocular slam supported object recognition,''
  \emph{arXiv preprint arXiv:1506.01732}, 2015.

\bibitem{pire2019online}
T.~Pire, J.~Corti, and G.~Grinblat, ``Online object detection and localization
  on stereo visual slam system,'' \emph{Journal of Intelligent \& Robotic
  Systems}, pp. 1--10, 2019.

\bibitem{atanasov2014semantic}
N.~Atanasov, M.~Zhu, K.~Daniilidis, and G.~J. Pappas, ``Semantic localization
  via the matrix permanent.'' in \emph{Robotics: Science and Systems}, vol.~2,
  2014.

\bibitem{reid2014towards}
I.~Reid, ``{Towards semantic visual SLAM},'' in \emph{International Conference
  on Control Automation Robotics \& Vision}.\hskip 1em plus 0.5em minus
  0.4em\relax IEEE, 2014.

\bibitem{kundu2014joint}
A.~Kundu, Y.~Li, F.~Dellaert, F.~Li, and J.~M. Rehg, ``Joint semantic
  segmentation and 3d reconstruction from monocular video,'' in \emph{European
  Conference on Computer Vision}.\hskip 1em plus 0.5em minus 0.4em\relax
  Springer, 2014, pp. 703--718.

\bibitem{galvez2016real}
D.~G{\'a}lvez-L{\'o}pez, M.~Salas, J.~D. Tard{\'o}s, and J.~Montiel,
  ``{Real-time monocular object SLAM},'' \emph{Robotics and Autonomous
  Systems}, vol.~75, pp. 435--449, 2016.

\bibitem{doherty2020probabilistic}
K.~J. Doherty, D.~P. Baxter, E.~Schneeweiss, and J.~J. Leonard, ``Probabilistic
  data association via mixture models for robust semantic slam,'' in \emph{2020
  IEEE International Conference on Robotics and Automation (ICRA)}.\hskip 1em
  plus 0.5em minus 0.4em\relax IEEE, 2020, pp. 1098--1104.

\bibitem{rosinol2020kimera}
A.~Rosinol, M.~Abate, Y.~Chang, and L.~Carlone, ``Kimera: an open-source
  library for real-time metric-semantic localization and mapping,'' in
  \emph{2020 IEEE International Conference on Robotics and Automation
  (ICRA)}.\hskip 1em plus 0.5em minus 0.4em\relax IEEE, 2020, pp. 1689--1696.

\bibitem{Krishna_ICRA2017}
J.~K. {Murthy}, G.~V.~S. {Krishna}, F.~{Chhaya}, and K.~M. {Krishna},
  ``{Reconstructing vehicles from a single image: Shape priors for road scene
  understanding},'' in \emph{IEEE International Conference on Robotics and
  Automation (ICRA)}, May 2017, pp. 724--731.

\bibitem{parkhiya2018constructing}
P.~Parkhiya, R.~Khawad, J.~K. Murthy, B.~Bhowmick, and K.~M. Krishna,
  ``{Constructing Category-Specific Models for Monocular Object-SLAM},'' in
  \emph{IEEE International Conference on Robotics and Automation (ICRA)}, 2018,
  pp. 4517--4524.

\bibitem{salas2013slam++}
R.~F. {Salas-Moreno}, R.~A. {Newcombe}, H.~{Strasdat}, P.~H.~J. {Kelly}, and
  A.~J. {Davison}, ``{SLAM++: Simultaneous Localisation and Mapping at the
  Level of Objects},'' in \emph{IEEE Conference on Computer Vision and Pattern
  Recognition (CVPR)}, June 2013, pp. 1352--1359.

\bibitem{semslam}
S.~L. {Bowman}, N.~{Atanasov}, K.~{Daniilidis}, and G.~J. {Pappas},
  ``Probabilistic data association for semantic slam,'' in \emph{IEEE
  International Conference on Robotics and Automation (ICRA)}, May 2017, pp.
  1722--1729.

\bibitem{Atanasov_SemanticSLAM_IJCAI18}
N.~Atanasov, S.~L. Bowman, K.~Daniilidis, and G.~J. Pappas, ``A unifying view
  of geometry, semantics, and data association in slam,'' in
  \emph{International Joint Conference on Artificial Intelligence (IJCAI)},
  July 2018, pp. 5204--5208.

\bibitem{dongFS17}
J.~{Dong}, X.~{Fei}, and S.~{Soatto}, ``Visual-inertial-semantic scene
  representation for 3d object detection,'' in \emph{IEEE Conference on
  Computer Vision and Pattern Recognition (CVPR)}, July 2017, pp. 3567--3577.

\bibitem{fei2018visual}
X.~Fei and S.~Soatto, ``{Visual-Inertial Object Detection and Mapping},'' in
  \emph{European Conference on Computer Vision (ECCV)}, 2018, pp. 301--317.

\bibitem{hu2018dense}
L.~Hu, Y.~Cao, P.~Wu, and L.~Kneip, ``Dense object reconstruction from rgbd
  images with embedded deep shape representations,'' \emph{arXiv preprint
  arXiv:1810.04891}, 2018.

\bibitem{feng2019localization}
Q.~Feng, Y.~Meng, M.~Shan, and N.~Atanasov, ``Localization and mapping using
  instance-specific mesh models,'' in \emph{IEEE/RSJ International Conference
  on Intelligent Robots and Systems (IROS)}, 2019, pp. 4985--4991.

\bibitem{hu2019deep}
L.~Hu, W.~Xu, K.~Huang, and L.~Kneip, ``{Deep-SLAM++: Object-level RGBD SLAM
  based on class-specific deep shape priors},'' \emph{arXiv
  preprint:1907.09691}, 2019.

\bibitem{ishimtsev2020cad}
V.~Ishimtsev, A.~Bokhovkin, A.~Artemov, S.~Ignatyev, M.~Niessner, D.~Zorin, and
  E.~Burnaev, ``Cad-deform: Deformable fitting of cad models to 3d scans,''
  \emph{arXiv preprint arXiv:2007.11965}, 2020.

\bibitem{frost2018recovering}
D.~{Frost}, V.~{Prisacariu}, and D.~{Murray}, ``Recovering stable scale in
  monocular slam using object-supplemented bundle adjustment,'' \emph{IEEE
  Transactions on Robotics}, vol.~34, no.~3, pp. 736--747, June 2018.

\bibitem{okhierarchical}
K.~Ok, K.~Liu, and N.~Roy, ``Hierarchical object map estimation for efficient
  and robust navigation,'' in \emph{International Conference on Robotics and
  Automation (ICRA)}.\hskip 1em plus 0.5em minus 0.4em\relax IEEE, May 2021.

\bibitem{bao2011semantic}
S.~Y. Bao and S.~Savarese, ``Semantic structure from motion,'' in \emph{CVPR
  2011}.\hskip 1em plus 0.5em minus 0.4em\relax IEEE, 2011, pp. 2025--2032.

\bibitem{cubeslam}
S.~{Yang} and S.~{Scherer}, ``{CubeSLAM: Monocular 3-D Object SLAM},''
  \emph{IEEE Transactions on Robotics}, vol.~35, no.~4, pp. 925--938, Aug 2019.

\bibitem{dhiman2016continuous}
V.~Dhiman, Q.-H. Tran, J.~J. Corso, and M.~Chandraker, ``A continuous occlusion
  model for road scene understanding,'' in \emph{Proceedings of the IEEE
  Conference on Computer Vision and Pattern Recognition}, 2016, pp. 4331--4339.

\bibitem{rubino20173d}
C.~Rubino, M.~Crocco, and A.~Del~Bue, ``3d object localisation from multi-view
  image detections,'' \emph{IEEE transactions on pattern analysis and machine
  intelligence}, vol.~40, no.~6, pp. 1281--1294, 2017.

\bibitem{GAY2018124}
P.~Gay, C.~Rubino, M.~Crocco, and A.~D. Bue, ``Factorization based structure
  from motion with object priors,'' \emph{Computer Vision and Image
  Understanding}, vol. 172, pp. 124 -- 137, 2018.

\bibitem{quadric_slam}
L.~{Nicholson}, M.~{Milford}, and N.~{Sünderhauf}, ``Quadricslam: Dual
  quadrics from object detections as landmarks in object-oriented slam,''
  \emph{IEEE Robotics and Automation Letters}, vol.~4, no.~1, pp. 1--8, Jan
  2019.

\bibitem{hosseinzadeh2018structure}
M.~Hosseinzadeh, Y.~Latif, T.~Pham, N.~Suenderhauf, and I.~Reid, ``Structure
  aware slam using quadrics and planes,'' in \emph{Asian Conference on Computer
  Vision}.\hskip 1em plus 0.5em minus 0.4em\relax Springer, 2018, pp. 410--426.

\bibitem{ok2019robust}
K.~Ok, K.~Liu, K.~Frey, J.~P. How, and N.~Roy, ``Robust object-based slam for
  high-speed autonomous navigation,'' in \emph{2019 International Conference on
  Robotics and Automation (ICRA)}.\hskip 1em plus 0.5em minus 0.4em\relax IEEE,
  2019, pp. 669--675.

\bibitem{wu2020eao}
Y.~Wu, Y.~Zhang, D.~Zhu, Y.~Feng, S.~Coleman, and D.~Kerr, ``{EAO-SLAM:
  Monocular Semi-Dense Object SLAM Based on Ensemble Data Association},''
  \emph{arXiv preprint arXiv:2004.12730}, 2020.

\bibitem{liao2020object}
Z.~Liao, W.~Wang, X.~Qi, X.~Zhang, L.~Xue, J.~Jiao, and R.~Wei,
  ``Object-oriented slam using quadrics and symmetry properties for indoor
  environments,'' \emph{arXiv preprint arXiv:2004.05303}, 2020.

\bibitem{savarese20073d}
S.~Savarese and L.~Fei-Fei, ``3d generic object categorization, localization
  and pose estimation,'' in \emph{2007 IEEE 11th International Conference on
  Computer Vision}.\hskip 1em plus 0.5em minus 0.4em\relax IEEE, 2007, pp.
  1--8.

\bibitem{quadrics_reid}
M.~{Hosseinzadeh}, K.~{Li}, Y.~{Latif}, and I.~{Reid}, ``Real-time monocular
  object-model aware sparse slam,'' in \emph{International Conference on
  Robotics and Automation (ICRA)}, May 2019, pp. 7123--7129.

\bibitem{orcvio}
M.~{Shan}, Q.~{Feng}, and N.~{Atanasov}, ``{OrcVIO: Object residual constrained
  Visual-Inertial Odometry},'' in \emph{IEEE Intl. Conf. on Intelligent Robots
  and Systems (IROS)}, 2020.

\bibitem{BarfootBook}
T.~D. Barfoot, \emph{State Estimation for Robotics}.\hskip 1em plus 0.5em minus
  0.4em\relax Cambridge University Press, 2017.

\bibitem{MVGBook}
R.~Hartley and A.~Zisserman, \emph{Multiple view geometry in computer
  vision}.\hskip 1em plus 0.5em minus 0.4em\relax Cambridge university press,
  2003.

\bibitem{FAST}
E.~{Rosten}, R.~{Porter}, and T.~{Drummond}, ``Faster and better: A machine
  learning approach to corner detection,'' \emph{IEEE Transactions on Pattern
  Analysis and Machine Intelligence}, vol.~32, no.~1, pp. 105--119, Jan 2010.

\bibitem{redmon2018yolov3}
J.~Redmon and A.~Farhadi, ``Yolov3: An incremental improvement,'' \emph{arXiv
  preprint arXiv:1804.02767}, 2018.

\bibitem{zhou2018starmap}
X.~Zhou, A.~Karpur, L.~Luo, and Q.~Huang, ``Starmap for category-agnostic
  keypoint and viewpoint estimation,'' in \emph{Computer Vision -- ECCV}, 2018,
  pp. 328--345.

\bibitem{LK}
B.~D. Lucas and T.~Kanade, ``An iterative image registration technique with an
  application to stereo vision,'' in \emph{Int. Joint Conf. on Artificial
  Intelligence (IJCAI)}, 1981, pp. 674--679.

\bibitem{troiani20142}
C.~Troiani, A.~Martinelli, C.~Laugier, and D.~Scaramuzza, ``2-point-based
  outlier rejection for camera-imu systems with applications to micro aerial
  vehicles,'' in \emph{2014 IEEE international conference on robotics and
  automation (ICRA)}.\hskip 1em plus 0.5em minus 0.4em\relax IEEE, 2014, pp.
  5530--5536.

\bibitem{gal2015dropout}
Y.~Gal and Z.~Ghahramani, ``{Dropout as a Bayesian Approximation: Representing
  Model Uncertainty in Deep Learning},'' in \emph{International Conference on
  Machine Learning (ICML)}, vol.~48, Jun 2016, pp. 1050--1059.

\bibitem{bewley2016simple}
A.~{Bewley}, Z.~{Ge}, L.~{Ott}, F.~{Ramos}, and B.~{Upcroft}, ``Simple online
  and realtime tracking,'' in \emph{IEEE International Conference on Image
  Processing (ICIP)}, Sep 2016, pp. 3464--3468.

\bibitem{PNP}
J.~A. {Hesch} and S.~I. {Roumeliotis}, ``A direct least-squares (dls) method
  for pnp,'' in \emph{International Conference on Computer Vision (ICCV)}, Nov
  2011, pp. 383--390.

\bibitem{yang2020teaser}
H.~Yang, J.~Shi, and L.~Carlone, ``{TEASER: Fast and Certifiable Point Cloud
  Registration},'' \emph{IEEE Transactions on Robotics}, vol.~37, no.~2, pp.
  314--333, 2021.

\bibitem{Kabsch}
W.~Kabsch, ``A discussion of the solution for the best rotation to relate two
  sets of vectors,'' \emph{Acta Crystallographica Section A}, vol.~34, no.~5,
  pp. 827--828, 1978.

\bibitem{quatekf}
N.~Trawny and S.~Roumeliotis, ``{Indirect Kalman Filter for 3D Attitude
  Estimation},'' University of Minnesota, Dept. of Comp. Sci. \& Eng., Tech.
  Rep, Tech. Rep., 2005.

\bibitem{sola2012quaternion}
J.~Sola, ``{Quaternion kinematics for the error-state KF},'' \emph{Laboratoire
  dAnalyse et dArchitecture des Systemes-Centre national de la recherche
  scientifique (LAAS-CNRS), Toulouse, France, Tech. Rep}, 2012.

\bibitem{geiger2012we}
A.~Geiger, P.~Lenz, and R.~Urtasun, ``{Are we ready for Autonomous Driving? The
  KITTI Vision Benchmark Suite},'' in \emph{IEEE Conference on Computer Vision
  and Pattern Recognition (CVPR)}, 2012, pp. 3354--3361.

\bibitem{sturm2012benchmark}
J.~Sturm, N.~Engelhard, F.~Endres, W.~Burgard, and D.~Cremers, ``A benchmark
  for the evaluation of rgb-d slam systems,'' in \emph{2012 IEEE/RSJ
  International Conference on Intelligent Robots and Systems}.\hskip 1em plus
  0.5em minus 0.4em\relax IEEE, 2012, pp. 573--580.

\bibitem{Zhang18iros}
Z.~{Zhang} and D.~{Scaramuzza}, ``A tutorial on quantitative trajectory
  evaluation for visual(-inertial) odometry,'' in \emph{IEEE/RSJ International
  Conference on Intelligent Robots and Systems (IROS)}, Oct 2018, pp.
  7244--7251.

\bibitem{mousavian20173d}
A.~{Mousavian}, D.~{Anguelov}, J.~{Flynn}, and J.~{Košecká}, ``{3D Bounding
  Box Estimation Using Deep Learning and Geometry},'' in \emph{IEEE Conference
  on Computer Vision and Pattern Recognition (CVPR)}, July 2017, pp.
  5632--5640.

\bibitem{xiang2017subcategory}
Y.~{Xiang}, W.~{Choi}, Y.~{Lin}, and S.~{Savarese}, ``Subcategory-aware
  convolutional neural networks for object proposals and detection,'' in
  \emph{IEEE Winter Conference on Applications of Computer Vision (WACV)},
  March 2017, pp. 924--933.

\bibitem{geneva2019openvins}
P.~Geneva, K.~Eckenhoff, W.~Lee, Y.~Yang, and G.~Huang, ``Openvins: A research
  platform for visual-inertial estimation,'' in \emph{IROS 2019 Workshop on
  Visual-Inertial Navigation: Challenges and Applications}, 2019.

\bibitem{xiaochen2020lightweight}
Q.~Xiaochen, Z.~Hai, and F.~Wenxing, ``Lightweight hybrid visual-inertial
  odometry with closed-form zero velocity update,'' \emph{Chinese Journal of
  Aeronautics}, 2020.

\end{thebibliography}
